\documentclass[10pt,twocolumn,letterpaper]{article}

\usepackage[pagenumbers]{wacv} 

\newcounter{oldtocdepth}

\definecolor{wacvblue}{rgb}{0.21,0.49,0.74}
\usepackage[pagebackref,breaklinks,colorlinks,allcolors=wacvblue]{hyperref}
\usepackage{graphicx}
\usepackage{booktabs}

\usepackage[ruled,vlined]{algorithm2e}
\usepackage{bbm}
\usepackage{amsthm}

\newtheorem{definition}{Definition}
\newtheorem{proposition}{Proposition}
\newtheorem{corollary}{Corollary}
\newtheorem{lemma}{Lemma}
\theoremstyle{definition}
\newtheorem{remark}{Remark}
\newtheorem{theorem}{Theorem}

\def\wacvPaperID{94}
\def\confName{WACV}
\def\confYear{2027}

\usepackage{marginnote}

\usepackage{soul}
\usepackage{tikz}

\newcommand\encircle[1]{%
  \tikz[baseline=(X.base)] 
    \node (X) [draw, shape=circle, inner sep=0] {\strut #1};}
\title{Batch Normalization Amplifies Memorization and Privacy Risks}

\author{Ngoc Phu Doan\\
Queen's University Belfast\\
{\tt\small ndoan01@qub.ac.uk}
\and
Chongyan Gu\\
Queen's University Belfast\\
{\tt\small c.gu@qub.ac.uk}
\and
Ihsen Alouani\\
Queen's University Belfast\\
{\tt\small i.alouani@qub.ac.uk}
}

\begin{document}
\maketitle

\begin{abstract}
Batch Normalization (BN) is widely adopted to enable faster convergence and more stable training of deep neural networks. However, its impact on privacy and memorization has remained largely unexplored. In this work, we investigate the effect of BN layers on the memorization of atypical or outlier samples and its implications for privacy leakage. We conduct an extensive empirical study using three complementary approaches: (i) unintended memorization of out-of-distribution samples, (ii) per-sample influence, and (iii) susceptibility to membership inference attacks (MIA). Across multiple datasets and architectures, we consistently observe that BN substantially increases the memorization of outliers compared to models without BN. Critically, this amplified memorization translates directly into privacy vulnerabilities: models with BN exhibit significantly higher susceptibility to MIAs. We complement our empirical findings with a mechanistic analysis under the exact BN backward pass, which shows that BN amplifies the per-step margin growth of outlier samples during training. Our results highlight an underappreciated privacy risk associated with BN and provide both practical and theoretical insights into how normalization layers can amplify the influence of rare or sensitive training examples.
\end{abstract}
\section{Introduction}
\label{sec:introduction}
Batch normalization (BN) has become a cornerstone in modern deep learning architectures, normalizing activations across intermediate layers to enhance training stability and overall model performance.  It has shown consistent improvement in models' accuracy and training convergence speed, making it a fundamental component of state-of-the-art architectures. Nevertheless, despite its well-established advantages and pervasive use, the deeper implications of BN on the behavior and properties of deep neural networks remain poorly understood. The original motivation for BN, reducing internal covariance shift \cite{ioffe2015batch}, has been disputed \cite{santurkar2018does}. The hypothesis that normalization layers enhance smoothness is supported through both empirical and theoretical analyses \cite{santurkar2018does,bjorck2018understanding,lyu2022understanding}, while the Hessian analysis in \cite{9378171} disputes this hypothesis and shows that BN does not necessarily make the loss landscape smooth.

Beyond optimization dynamics, recent work has explored BN's implications for model behavior. For example, several works have investigated the impact of BN on the models' robustness to adversarial perturbations \cite{wang2022removing, zhang2022achieving, Kong2023BatchNormVulnerability}, eventually concluding a negative impact of BN with an increasing vulnerability to adversarial noise. This suggests that BN might have unintended side effects, influencing the way models represent and respond to data. While robustness has attracted substantial attention, the \emph{privacy implications} of BN remain largely overlooked. In particular, the relationship between BN and \emph{memorization} has not been systematically studied. Memorization constitutes a primary avenue for privacy leakage, underpinning attacks like MIA~\cite{shokri2017membership, carlini2019secret} and sensitive information extraction \cite{carlini2023extracting}. Understanding whether BN accelerates memorization is therefore critical in assessing its broader security consequences.

\noindent\textbf{Motivation: Why Outlier Memorization Matters?} Prior work has established that memorization in deep learning is not uniformly distributed across the training set: outliers, rare samples, and mislabeled data are \textit{disproportionately} likely to be memorized \cite{feldman20, carlini2022membership, 10.5555/3600270.3601234}. From a privacy perspective, this phenomenon is particularly concerning because outliers often contain sensitive information, such as individuals with rare medical conditions, minority demographic characteristics, or unique behavioral patterns, precisely the cases where privacy protection is most critical.  In this work, we present the first comprehensive study that establishes a link between BN and memorization and privacy leakage, particularly for outliers and rare samples. Specifically, our results demonstrate that BN both accelerates memorization of these vulnerable samples and amplifies their leakage under MIAs. This raises a critical concern: if BN further emphasizes the very samples that are most vulnerable to memorization, its widespread use in modern architectures may inadvertently increase privacy risks.

\noindent \textbf{Contributions.} We make the following contributions: 
\begin{itemize}
    \item On the empirical side, we provide a systematic evaluation of memorization under three complementary lenses: (\emph{i}) controlled label-flip and out-of-distribution (OOD) experiments that test the model’s propensity to memorize outliers, (\emph{ii}) per-sample influence analysis via LOO that reveals the sensitivity of models to individual examples, and (\emph{iii}) MIA that serve as an auditing tool to quantify privacy risk. Together, these approaches offer a multifaceted view of memorization in BN versus no-BN models.
    \item We found that BN consistently leads to \textbf{faster and stronger memorization} of outliers, as well as \textbf{increased leakage} under membership inference attacks.
    
    \item On the theoretical side, we provide a rationale why BN disproportionately accelerates the memorization of outlier or noisy samples.  Working with the \emph{exact} BN backward pass (i.e., differentiating through the batch statistics rather than freezing them), we prove that BN amplifies per-step margin growth on tail or mislabeled samples by a factor of $(\gamma/\sigma)^{2}\kappa_{\mathrm{eff}}^{2}$, where $\gamma$ is BN's learnable scale, $\sigma$ is the \emph{batch} standard deviation of the pre-activations, and $\kappa_{\mathrm{eff}}\in(0,1]$ is a geometric factor introduced by the exact backward. We further show \emph{why} $\gamma/\sigma$ grows: $\gamma$ is driven by the unfit tail while $\sigma$ is set by the bulk, so the ratio rises in proportion to the memorization load, and we observe $(\gamma/\sigma)>1$ across a wide range of trained models. This quadratic amplification is particularly pronounced for channels with low variance relative to their learned scale parameter. 
    \item  We introduce a BN parameter regularization algorithm (\textbf{BNReg}) which constrains the ratio $(\gamma/\sigma)$ during training that sufficiently mitigates memorization on outliers while retaining generalizability.

\end{itemize}
\noindent These results reveal a fundamental tension between optimization benefits and privacy risks. Our findings suggest practitioners should carefully weigh these tradeoffs when deploying BN in privacy-sensitive applications.
\section{Background}
\label{sec:preliminaries}

\noindent \textbf{Batch Normalization.} BN has been widely used in Deep Neural Networks since it accelerates training convergence and reduces the internal covariance shift to stabilize learning dynamics \cite{ioffe2015batch}. BN normalizes the data representation space of intermediate layers to a zero-mean distribution to strengthen the discrimination among data points, thus making it easier for the succeeding layers to capture patterns. Let $\mathcal{B}_l=\{z_{1...m}^l\}$ be the mini-batch output of the model at the $l^{th}$ layer. Given a batch size $m$, mini-batch mean and variance are computed during training as $\mu_{\mathcal{B}_l}=1/m\sum_{i=1}^{m} z_i$ and $\sigma_{\mathcal{B}_l}^2=1/m\sum_{i=1}^{m}(z_i-\mu_{\mathcal{B}_l})^2$. BN at $l^{th}$ layer is determined by:
\begin{equation}
    y_i=BN(\mathcal{B}_l)=\gamma \frac{z_i-\mu_{\mathcal{B}_l}}{\sqrt{\sigma_{\mathcal{B}_l}^2+\epsilon}} + \beta,
\end{equation}
where $\gamma$, $\beta$ are the scale and shift learnable parameters and $\epsilon > 0$ is an arbitrarily small number to avoid division by 0. 

\noindent\textbf{Memorization.} In machine learning, \emph{memorization} refers to the phenomenon where a model’s prediction on a training example depends directly on that specific example being present in the training set, rather than on generalizable patterns shared across the data distribution~\cite{feldman20}.
Prior work shows that \emph{outliers, rare, noisy, or mislabeled examples} are disproportionately prone to memorization. For instance,~\cite{carlini2023extracting} demonstrated that memorized content is often concentrated in atypical or low-probability regions of the input space. This tail sensitivity underlies both the generalization-memorization tradeoff~\cite{feldman20} and the privacy risks of large-scale models~\cite{carlini2019secret}.

\noindent\textbf{Membership Inference Attack.} MIA is a privacy attack where an adversary aims to determine whether a particular data point was included in the training set of a model~\cite{shokri2017membership, carlini2022membership}. Successful MIA violates data privacy: it reveals that an individual's information was used for training, which can have serious consequences in domains such as healthcare (revealing medical history), finance (exposing transaction records), or personal data (disclosing user behavior). MIA exploits a fundamental asymmetry: trained models often exhibit lower loss on training samples compared to unseen data, creating a distinguishable signal that adversaries can use.

Beyond being an attack vector, MIA serves as a \emph{privacy auditing tool} that quantifies the practical privacy risk introduced by memorization. The success rate of MIA, therefore, provides a direct measure of privacy leakage: higher attack accuracy indicates potential risk of sensitive information extraction about individuals in the training set~\cite{song2021systematic}.

\section{Methodology}
\label{sec:methodology}
\noindent\textbf{Research Question.} 
Memorization in deep learning disproportionately affects outliers, rare samples, and atypical examples \cite{feldman20, carlini2022privacy}. This poses severe privacy risks because such samples often represent individuals with sensitive characteristics (rare medical conditions, minority demographics, etc.) and are the primary targets of successful MIAs and data extraction attacks \cite{carlini2022membership}. Given BN's established role in accelerating optimization \cite{ioffe2015batch}, a natural question arises: \textit{Does the same mechanism that accelerates learning on typical samples also accelerate memorization of outliers?} If so, BN may inadvertently amplify privacy risks for the most vulnerable training examples.


\noindent \textbf{Problem formulation.} Let $\mathcal{D}$ be the data population distribution. We refer to the learning algorithm $\mathcal{A}$ on the training data $S\sim\mathcal{D}^n$ that learns a probabilistic function $f:\mathbb{R}^d\rightarrow \mathbb{R}^\mathcal{C}$, where $d, \mathcal{C}$ denote the dimension of the input space and the number of classes, respectively. We assume that $S$ contains an atypical set (outliers, rare samples) $S^{\text{out}}$ and a clean set $S^{\text{in}}$ or, equivalently, $S=S^{\text{in}}\bigcup S^{\text{out}}$. Let $k$ denote the ratio of outliers in the training set $S$; i.e., $k=|S^{\text{out}}|/|S^{\text{in}}|$. Let $\mathcal{M}(f, S)$ be a given $\textit{memorization measure}$, which quantifies how much memorization the model $f$ exhibits on the training data $S$. There are several definitions of memorization in the literature and therefore different metrics:

\noindent \textbf{i)} Unintended Memorization: A model exhibits unintended memorization where its prediction on a training example $(x_i, y_i)$ depends on that specific example being in the training set rather than on generalizable patterns \cite{carlini2019secret, tirumala2022memorization}.

\noindent \textbf{ii)} Sample Influence: Memorization manifests through the disproportionate influence that certain samples exert on model parameters during training \cite{feldman20, zhang2023counterfactual}.



\begin{figure}[t]
    \centering
    \includegraphics[width=\linewidth]{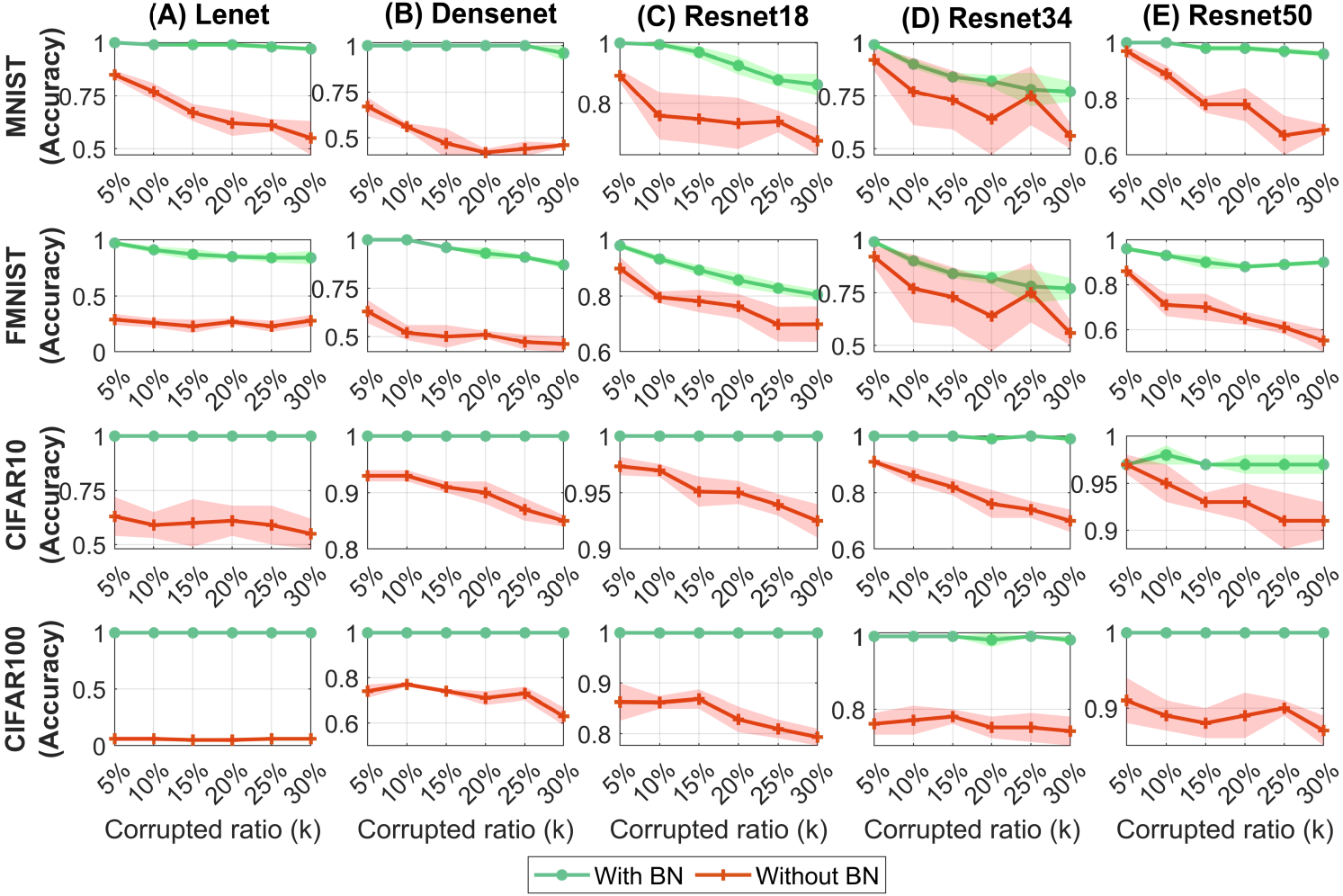}
    \caption{Memorization of models with BN (green lines) and no-BN (red lines) across diverse label-corrupted datasets, including MNIST, FashionMNIST, CIFAR10, and CIFAR100.}
    \label{fig:hard_memorization}
\end{figure}

\noindent\textbf{Experimental methodology.} We investigate BN's impact on memorization through three complementary lenses, each designed to capture different aspects of how models internalize training data:
\noindent\textbf{\encircle{a} Forced Memorization via Label Corruption and noisy data} (Sections~\ref{subsec:forced_memorization}). This experiment corresponds to the unintended memorization case. We examine whether BN accelerates memorization by training on samples $(x, \tilde{y})$ where labels $\tilde{y}$ are intentionally corrupted (i.e., $\tilde{y}\neq y$). Since no generalizable pattern exists between $x$ and $\tilde{y}$, high accuracy on these samples indicates pure memorization. We also assess the impact of BN on OOD data $(\tilde{x}, y)$ where $\tilde{x}$ is sampled from another distribution. This controlled setting allows us to quantify memorization capacity in isolation from generalization.
\noindent \textbf{\encircle{b} Mechanistic Analysis via leave-one-out} (Section~\ref{subsec:lood}) We estimate the model memorization through the sample influence by the leave-one-out (LOO) mechanism, inspired by prior works~\cite{feldman20, ye2024leave}.
\noindent \textbf{\encircle{c} Privacy Leakage via MIA} (Section~\ref{subsec:mia}). We quantify the distinguishability of training samples, using MIAs on naturally occurring OOD samples. These samples, identified through statistical outlier detection, represent real-world privacy risks where training data characteristics can be inferred. 

\noindent \textbf{Datasets and models.} We conduct experiments on four image classification datasets: MNIST \cite{deng2012mnist}, FashionMNIST \cite{xiao2017fashion}, CIFAR10 \cite{krizhevsky2009learning} and CIFAR100 \cite{krizhevsky2009learning}. For model architectures, we employ Lenet \cite{lecun2002gradient}, Densenet \cite{huang2017densely}, Resnet18, Resnet34 and Resnet50 \cite{he2016deep}. For forced memorization, we repeat experiments five times and report the averaged value to guarantee the robustness of our results. We additionally quantify the memorization of outliers on ViT \cite{dosovitskiy2020image} with Layer Norm \cite{ba2016layernormalization} (c.f. Section E.5 Supplementaries); more complex datasets such as SVHN \cite{netzer2011reading} and Tiny ImageNet \cite{huynh2022vision} (c.f. Section E.6 Supplementaries); and diverse normalization techniques such as Instance Norm \cite{ulyanov2016instance}, Group Norm \cite{wu2018group}, Power Norm \cite{shen2020powernorm}, and Norm Matters \cite{hoffer2018norm} (c.f. Section E.4 Supplementaries).

\noindent\textbf{Implementation details and Ablation study.} Details of our models' configuration and hyperparameter sensitivity can be found in Sections E and F of the Supplementaries.

\section{Experiments}
\label{sec:experiment}
\subsection{Forced memorization}
\label{subsec:forced_memorization}

\noindent\textbf{Methodology.} We conduct a memorization scheme where models are constrained to overfit data containing atypical samples, following setting \encircle{a} detailed in Section \ref{sec:methodology}. 
This forced memorization setting is commonly investigated in understanding generalization and memorization of deep neural networks so far \cite{zhang2016understanding}. 
In the out-of-distribution setting, new data points are sampled from other distributions, except that their labels still belong to the same label group of the original dataset (see Algorithms~B.1 and B.2 detailed in the Supplementary). The model architecture having a bigger marginal privacy risk on corrupted data indicates a higher capacity for data memorization.

\begin{figure}[t]
    \centering
    \includegraphics[width=\linewidth]{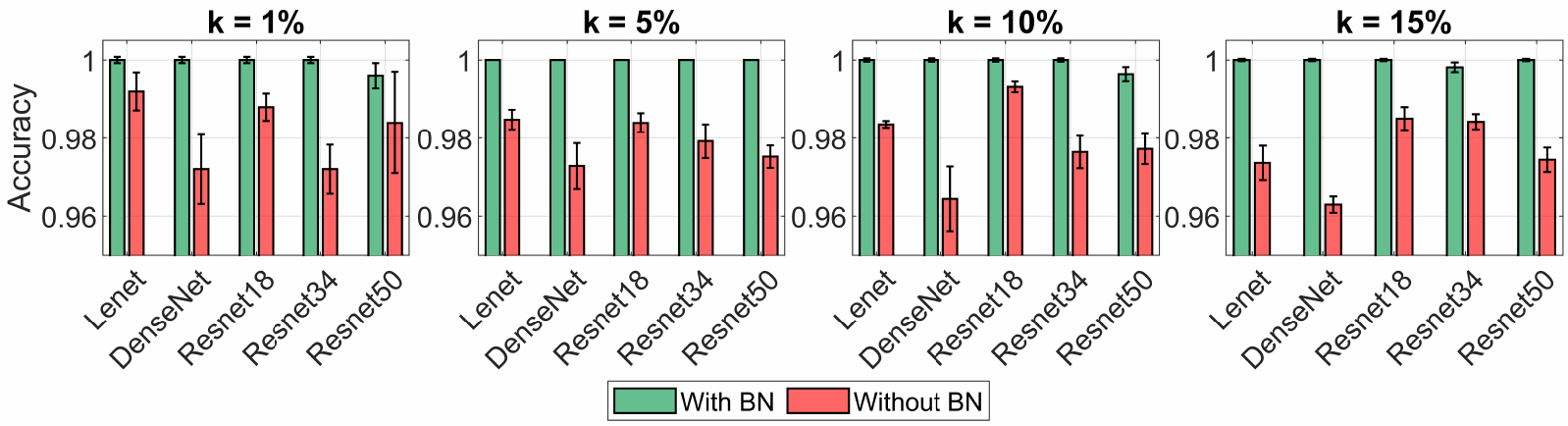}
    \vspace{-0mm}\caption{Accuracy of BN models versus no-BN models on out-of-distribution CIFAR10 datasets across different architectures.
    }\vspace{-0mm}
    \label{fig:ood_acc}
\end{figure}

\noindent\textbf{Results.} We inspect the behavior of BN models and no-BN models on datasets with flipped labels, as shown in Figure \ref{fig:hard_memorization}. We evaluate five architectures: Lenet (A), Densenet (B), Resnet18 (C), Resnet34 (D) and Resnet50 (E), across different flipping ratios $k$. The observation is that the models incorporating BN on the flipped-label portion (green curves) consistently achieve higher accuracy than models without BN (red curves) across all datasets and all ratios. For example, the corrupted accuracy of Resnet50 on MNIST can drop significantly between 3\% and 29\% when increasing the mislabeled ratio. This indicates that BN facilitates fitting noisy samples, thereby increasing the model's capacity for memorizing training data. 
Moreover, the margin of memorization is larger when increasing $k$, which can be seen clearly in CIFAR10. This trend implies that no-BN models diminish memorization on corrupted data when the data distribution becomes increasingly biased to the outlier distribution. In contrast, models with BN remain largely unaffected by this distributional shift, maintaining higher memorization capacity. 
Furthermore, Figure \ref{fig:ood_acc} compares models' accuracy on the OOD CIFAR10 across different architectures. While the models with BN perfectly fit the outlier data, the models without BN consistently show a weaker ability to fit OOD data, by around 3\% in terms of accuracy across varied values of $k$. Therefore, integrating BN layers into model architectures can improve the capacity to memorize OOD data distributions. 

\begin{figure}[t]
    \centering
    \includegraphics[width=\linewidth]{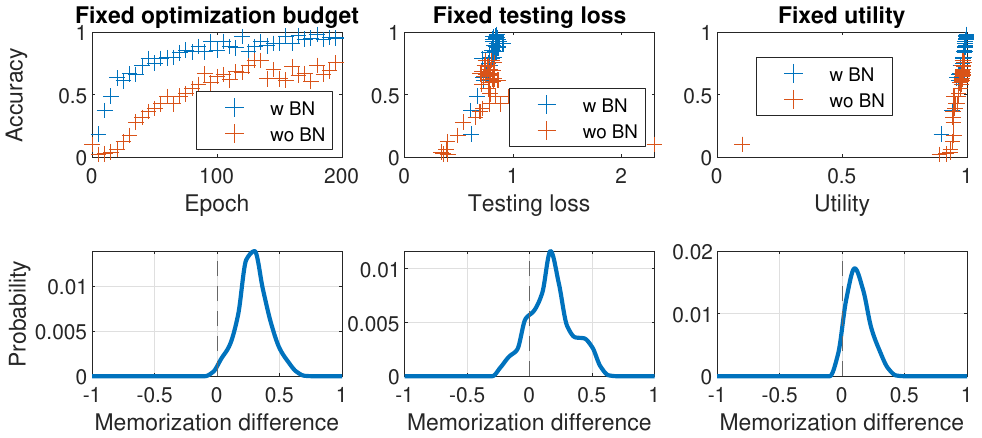}
    \caption{Memorization of models with and without BN across optimization budget, testing loss, and utility (first row). The memorization difference distribution between BN and no-BN models across the same budget range (second row).}
    \label{fig:memorization_control}
\end{figure}

\begin{figure}[t]
    \centering
    \includegraphics[width=1\linewidth]{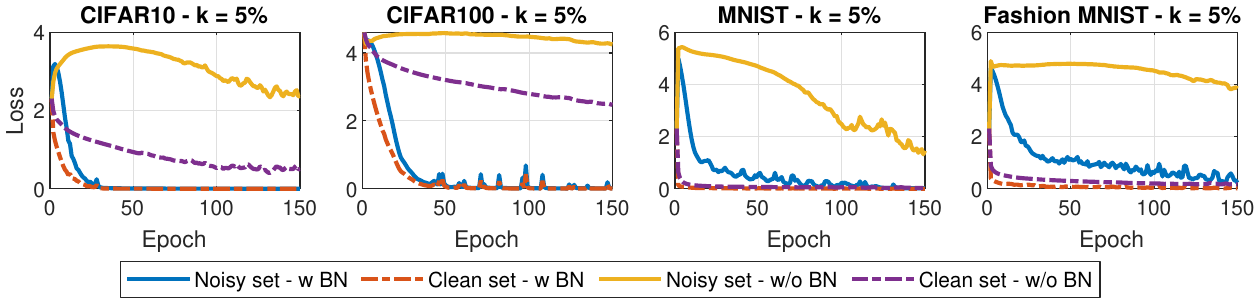}
    \vspace{-0mm}\caption{Loss values over epochs on MNIST, FashionMNIST, CIFAR10, CIFAR100. 
    }
    \label{fig:loss_analyse}
\end{figure}
\noindent\textbf{Memorization amplification under controlled budgets.} Our results on memorization in BN models and no-BN models under matched computational budgets, covering optimization, testing loss, and utility, are illustrated in Figure \ref{fig:memorization_control}. Specifically, we evaluate memorization by measuring the performance on corrupted labels across different checkpoints using two complementary views: (1) memorization of checkpoints across the full range of budgets and (2) memorization-amplified distribution obtained by binning checkpoints by budgets. BN models exhibit consistently higher memorization than no-BN models, and the amplified distribution skews toward the positive values, indicating BN models amplify memorization even when the computational budget is constant.

\noindent\textbf{Convergence speed analysis.} Figure \ref{fig:loss_analyse} visualizes loss trends by epochs of Lenet (with BN versus without BN) on four datasets with the ratio $k=5\%$. This analysis is conducted on both the noisy set and the clean set. Over datasets, the loss trend of the model incorporating BN on the noisy dataset (a solid blue line) dramatically approaches zero much faster than the trend of the model without BN (the solid yellow line). This observation indicates that besides increasing memorization capacity, BN can also accelerate memorizing atypical samples during learning.

\subsection{Leave-one-out influence}
\label{subsec:lood}
\noindent\textbf{Methodology.} We estimate the memorization through the sample influence by the LOO mechanism. Specifically, the influence of a sample $z$ on a model $f$ can be determined by $\text{mem}(z) = \text{dist}\Big(f_{S\sim \mathcal{D}^n, z\in S}(z), f_{S\sim \mathcal{D}^n, z\notin S}(z)\Big)$. Here we choose the Jensen-Shannon divergence as the distance to bound the memorization influence~\cite{hoyos2023representation}. A higher influence value indicates a higher memorization. We conduct experiments with a corruption rate of $10\%$. Instead of retraining once per point in the corrupted set, we train a modest number of models on independent subsets of $\mathcal{D}$. Then, each corrupted-label datapoint $z$ lands in some subsets and out of others; we, therefore, can compute the value of $\text{mem}(z)$. 

\noindent\textbf{Results.} Figure~\ref{fig:lood} shows that BN models are strongly influenced by noisy samples, as their influence-value distribution remains consistently narrow and concentrated in a high-value region. In contrast, the distributions for no-BN models are more spread out and clearly distinct from those of BN models. This concentrated influence pattern has important implications for privacy, as it suggests rare samples are disproportionately memorized in BN models.
\begin{figure}[t]
    \centering
    \includegraphics[width=\linewidth]{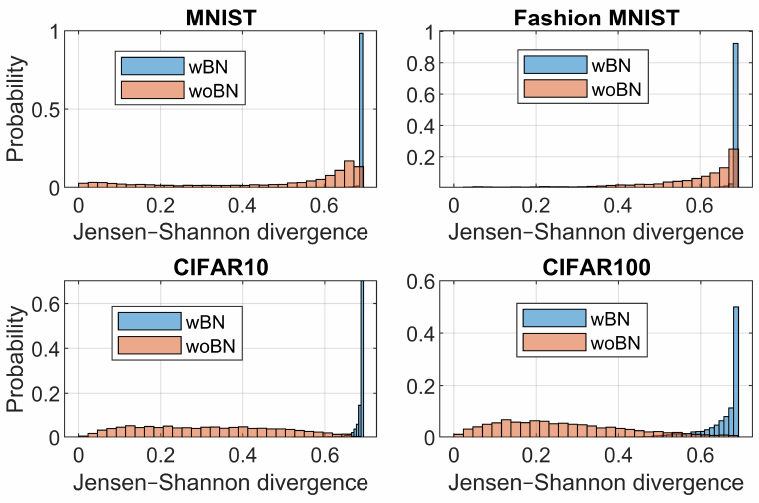}
    \caption{Effect of noisy samples on model memorization estimated by leave-one-out proxy.}\vspace{-5mm}
    \label{fig:lood}
\end{figure}


\subsection{Privacy Implication: Membership Inference}
\label{subsec:mia}
\noindent\textbf{Methodology.} To further investigate the role of BN in privacy, we conduct MIA through Likelihood Ratio Attack (LIRA) \cite{carlini2022membership}, which is one of the state-of-the-art MIAs to quantify models' privacy leakage. Specifically, given an arbitrary dataset $\mathcal{D}$, to mimic the behavior of the pretrained $f^*$,  LIRA introduces a collection of shadow models $f$, which have the same architecture as the target model $f^*$. The difference between models training with and without $x$ is estimated by a ratio: $\Lambda(f;x,y)=p(f; \mathbb{Q}_{in}(x,y))/p(f; \mathbb{Q}_{out}(x,y))$, where $\mathbb{Q}_{in}(x,y)\subset\mathcal{D}$, $(x, y)\in\mathbb{Q}_{in}(x,y)$ and $\mathbb{Q}_{out}(x,y)\subset\mathcal{D}$, $(x, y)\not\in\mathbb{Q}_{out}(x,y)$. Then, hypothesis testing is used on values of $\Lambda$ to make a decision if $(x, y)$ is in the target model's training data. We design an experiment to examine the direct impact of BN on the model's privacy. Specifically, we construct an out-of-distribution dataset of CIFAR10 dataset with the outlier ratio $k=1\%$ and $k=5\%$. 

\begin{figure}[t]
    \centering
    \includegraphics[width=\columnwidth]{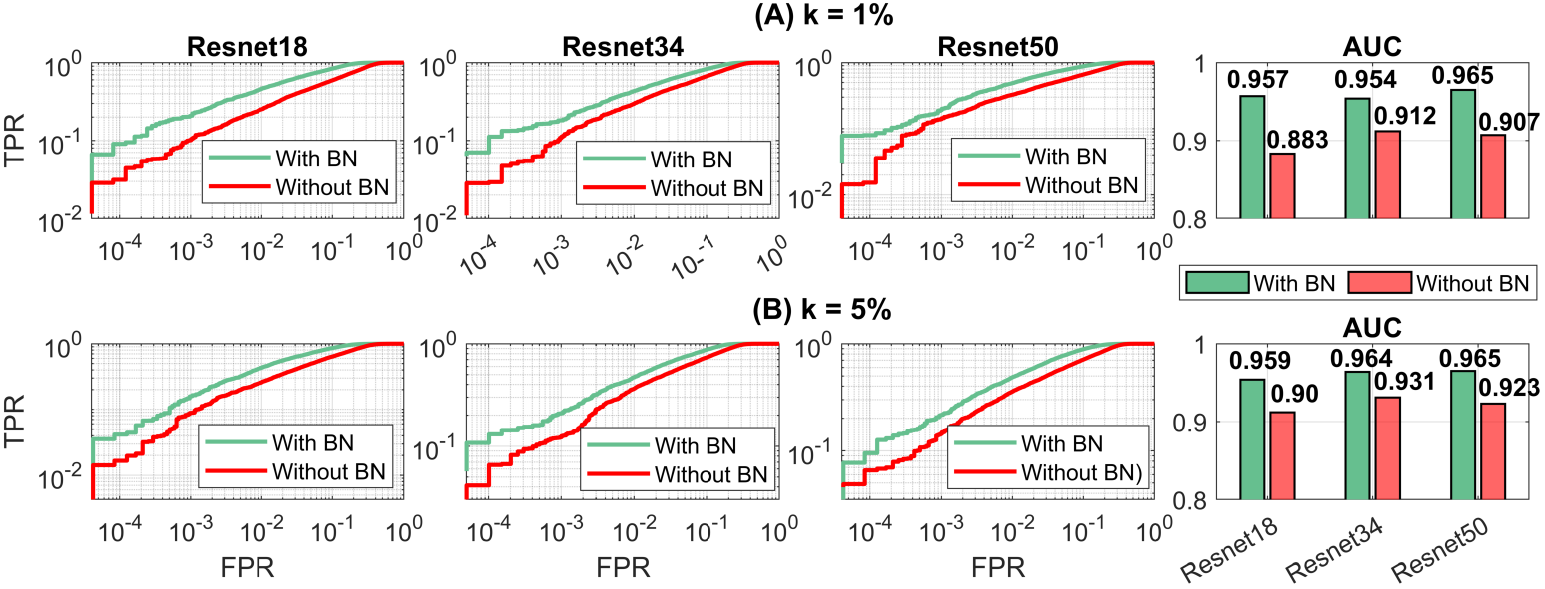}
    \caption{Comparison of attacking performance of LIRA across various DNNs, with and without BN, with a $1\%$-and $5\%$-corrupted label CIFAR10. (Best viewed in color)}
    \label{fig:mia_flip}
\end{figure}

\noindent\textbf{Results.}  We evaluate the privacy of three architectures: Resnet18, Resnet34 and Resnet50. The attack performance is reported in Figure \ref{fig:mia}. The result shows that, on average, the AUC of LIRA in models with BN layers is higher by at least $3\%$ for the $1\%$-atypical mixture and $4\%$ for the $5\%$-atypical mixture, compared to the models without BN. Furthermore, the ROC curves reveal that even at low False Positive Rate (FPR), the attack achieves consistently stronger performance against BN-based models. We further compare the performance of MIA on corrupted label CIFAR10, which is shown in Figure \ref{fig:mia_flip}. The observation shows that the attack performance on models incorporating BN is more vulnerable to MIA, with a margin of around $5\%$, illustrating an increase in privacy leakage. 
\section{Theoretical Analysis}
\label{sec:theory}

Our empirical results consistently show that BN amplifies the memorization of outliers. In this section we develop a theoretical account of \emph{why} this occurs. Unlike prior BN analyses that freeze the batch statistics, we work with the \emph{exact} BN backward pass, in which the gradient is routed through $\mu(w)$ and $\sigma(w)$. We derive the amplification factor exactly, and recover the fixed-statistics result as a large-batch corollary. Detailed proofs are in Section~A.1 of the Supplementary.
\paragraph{Setup.}
Consider a single convolutional channel computing pre-activations on a
mini-batch $\mathcal{B}=\{x_i\}_{i=1}^{B}\subset\mathbb{R}^{d}$,
\begin{equation}
h_i \;=\; w^{\top}x_i + b ,
\label{eq:preact}
\end{equation}
followed by a top linear classifier mapping the (possibly normalized) activation
to a logit $z_i$, with binary label $y_i\in\{+1,-1\}$ and logistic loss
$\ell_i=\log(1+e^{-y_i z_i})$. We compare two models that differ only in whether
BN is applied between the convolution and the classifier:
\begin{align}
\text{\textbf{no-BN:}}\quad
& z_i = a\,h_i + c,
\qquad s_{\mathrm{noBN}} = a,
\label{eq:nobn}\\[2pt]
\text{\textbf{BN:}}\quad
& u_i = \frac{h_i-\mu}{\sigma},
\qquad
\hat h_i = \gamma u_i+\beta,
\notag\\
& z_i = a\hat h_i+c,
\qquad
s_{\mathrm{BN}} = \frac{a\gamma}{\sigma}.
\label{eq:bn}
\end{align}
where $\mu=\tfrac1B\sum_i h_i$ and $\sigma^2=\tfrac1B\sum_i(h_i-\mu)^2$ are the
\emph{batch} statistics (functions of $w$), and $\gamma,\beta$ are BN's learnable
scale and shift. We write $\varphi(v)=1/(1+e^{-v})$ for the logistic sigmoid,
$m_i=y_iz_i$ for the margin, and
\begin{equation}
\rho_i \;:=\; -\,y_i\,\varphi(-m_i)
\label{eq:rho}
\end{equation}
for the per-sample logit residual, so that $\partial\ell_i/\partial z_i=\rho_i$.
Throughout, $L=\sum_i\ell_i$ and $X\in\mathbb{R}^{B\times d}$ has rows $x_i^\top$
with Gram matrix $K=XX^{\top}$.

\begin{figure}[t]
    \centering
    \includegraphics[width=\columnwidth]{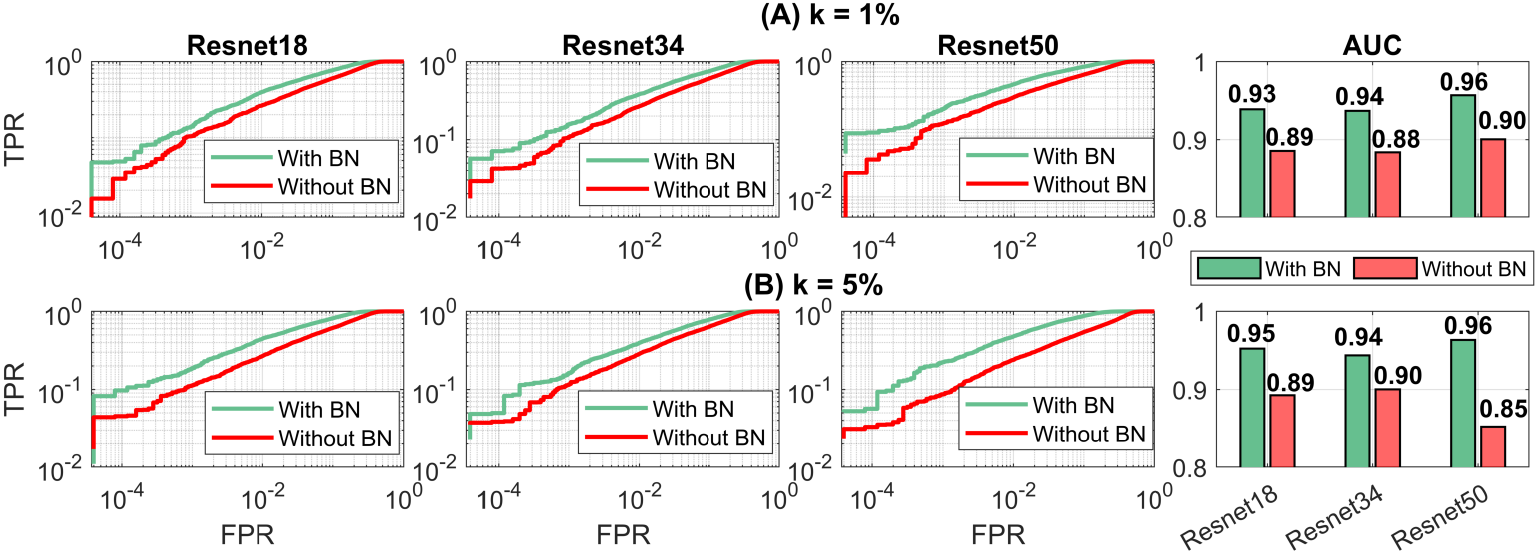}
    \caption{Comparison of attacking performance of LIRA across various DNNs, comparing models with BN versus without BN, with a $1\%$-atypical and $5\%$-atypical CIFAR10. Denote TPR (True Positive Rate), FPR (False Positive Rate) (Best view in color).}
    \label{fig:mia}
\end{figure}

\begin{definition}[Tail sample]
\label{def:tail}
A \emph{tail sample} $x^{\star}\in\mathcal{B}$ is one whose pre-activation lies
far from the batch mean, $h^{\star}=\mu+t\sigma$ with $|t|\gg 1$. We write
$t:=u_{\star}$ for its normalized activation and index it by $\star$.
\end{definition}

\subsection{The exact BN backward pass}
\label{sec:exact-backward}

The entire analysis rests on one structural fact, which we state first because it
governs both the magnitude and the \emph{direction} of the outlier's gradient.

\begin{lemma}[BN backward is an orthogonal projection]
\label{lem:projection}
Let $\mathbf{1}=(1,\dots,1)^{\top}$ and $u=(u_1,\dots,u_B)^{\top}$. Then
$\mathbf{1}^{\top}u=0$ and $\|u\|_2^2=B$, and the exact gradient of $L$ with
respect to the pre-activations is
\begin{equation}
\frac{\partial L}{\partial h}
\;=\;\frac{a\gamma}{\sigma}\,P\,\rho,
\qquad
P \;:=\; I-\tfrac1B\mathbf{1}\mathbf{1}^{\top}-\tfrac1B\,u\,u^{\top},
\label{eq:exact-backward}
\end{equation}
where $P$ is the \emph{orthogonal projector} onto
$\operatorname{span}\{\mathbf{1},u\}^{\perp}$, i.e.\ $P=P^{\top}=P^{2}$ and
$\|P\|_2=1$. Componentwise,
\begin{equation}
\begin{aligned}
\frac{\partial L}{\partial h_j}
&= \frac{a\gamma}{\sigma}
   \Big(\rho_j-\bar\rho-u_j\,\overline{\rho u}\Big),\\
\bar\rho &= \frac{1}{B}\sum_i\rho_i,
\qquad
\overline{\rho u}
= \frac{1}{B}\sum_i\rho_i u_i .
\end{aligned}
\label{eq:exact-backward-components}
\end{equation}
Moreover the same operator $\tfrac1\sigma P$ is the Jacobian
$\partial u/\partial h$, so BN applies $P$ on both the forward perturbation and
the backward pass.
\end{lemma}

The fixed-statistics model used in prior work retains only the first term,
$\partial L/\partial h_j\approx\frac{a\gamma}{\sigma}\rho_j$. The two discarded
terms are exactly the mean-subtraction and variance-coupling channels, and they
are the ones the reviewer of any fixed-statistics analysis should worry about.
Lemma~\ref{lem:projection} lets us quantify them in closed form.

Because $P$ is an orthogonal projector, its diagonal entries lie in $[0,1]$. The
diagonal entry at the outlier is the quantity that controls everything that
follows.

\begin{lemma}[Tail sample leverage]
\label{lem:leverage}
Define \emph{BN attenuation}
\begin{equation}
\kappa \;:=\; P_{\star\star} \;=\; 1-\frac{1+t^{2}}{B}.
\label{eq:kappa}
\end{equation}
Then the normalization constraints $\sum_i u_i=0$, $\sum_i u_i^2=B$ force
\begin{equation}
t^{2}\;\le\;B-1
\qquad\Longrightarrow\qquad
\kappa\in[0,1),
\label{eq:leverage-bound}
\end{equation}
with $\kappa=0$ attained only in the degenerate single-spike batch
$u=(\pm\sqrt{B-1},\mp\tfrac{1}{\sqrt{B-1}},\dots)$. For fixed $t$,
$\kappa\to1$ as $B\to\infty$.
\end{lemma}

Lemma~\ref{lem:leverage} shows that the outlier's share of the batch variance is exactly $t^{2}/B$, and BN's own normalization caps how extreme a single sample can be.

\subsection{Per-step margin amplification}

\begin{proposition}[Exact per-step margin growth]
\label{prop:exact-margin}
Fix a mini-batch containing a tail sample $(x^{\star},y^{\star})$ and take one
gradient descent step of size $\eta$ on the filter weights $w$, differentiating
through $\mu(w)$ and $\sigma(w)$. To first order in $\eta$, the outlier's margin
changes by
\begin{equation}
\begin{aligned}
\Delta m_{\star}
&= \eta\,\varphi(-m_{\star})
   \Big(\frac{a\gamma}{\sigma}\Big)^2
   (PKP)_{\star\star} \\
&= \eta\,\varphi(-m_{\star})
   \Big(\frac{a\gamma}{\sigma}\Big)^2
   \big\|X^{\top}p_{\star}\big\|_2^2
   \;\ge\;0.
\end{aligned}
\label{eq:exact-margin}
\end{equation}
where $p_{\star}=P e_{\star}$ has entries
$(p_\star)_\star=\kappa$ and $(p_\star)_j=-(1+u_jt)/B$ for $j\neq\star$, and where
the bulk is assumed well-classified ($\rho_j\approx0$, $j\neq\star$), as holds in
the memorization regime.
\end{proposition}

Two consequences deserve emphasis. First, $\Delta m_{\star}$ is a squared norm
and is therefore \emph{nonnegative for every batch composition}: routing the
gradient through $\mu$ and $\sigma$ cannot reverse the direction of margin
growth on the outlier. Second, the $(a\gamma/\sigma)^{2}$ prefactor survives the
exact treatment untouched; the entire effect of the previously neglected terms is absorbed into the geometric factor $(PKP)_{\star\star}$.

\begin{corollary}[Exact BN amplification factor]
\label{cor:amplification}
Under identical initialization and step size, the ratio of per-step margin growth between the BN and no-BN models on the same tail sample is
\begin{equation}
\mathcal{A}
\;=\;\frac{\Delta m_{\mathrm{BN}}}{\Delta m_{\mathrm{noBN}}}
\;=\;\Big(\frac{\gamma}{\sigma}\Big)^{2}\kappa_{\mathrm{eff}}^{2},
\qquad
\kappa_{\mathrm{eff}}^{2}
:=\frac{\|X^{\top}p_{\star}\|_2^{2}}{\|x^{\star}\|_2^{2}} .
\label{eq:amplification}
\end{equation}
If the tail sample is approximately orthogonal to the bulk ($\langle x_j,x^{\star}\rangle\approx0$ for $j\neq\star$, the generic situation for an outlier), then $\kappa_{\mathrm{eff}}^{2}\ge\kappa^{2}$, and if in addition the bulk contribution is negligible, $\kappa_{\mathrm{eff}}^{2}\to\kappa^{2}=\big(1-\tfrac{1+t^2}{B}\big)^{2}$.
\end{corollary}

\begin{remark}[Fixed-statistics model as the large-batch limit]
\label{cor:fixed-stats}
For fixed $t$ and $\|x^{\star}\|$, $\kappa\to1$ and
$\kappa_{\mathrm{eff}}\to1$ as $B\to\infty$, so \eqref{eq:exact-margin} reduces to
\begin{equation}
\frac{\Delta m_{\mathrm{BN}}}{\Delta m_{\mathrm{noBN}}}
=\frac{s_{\mathrm{BN}}^{2}}{s_{\mathrm{noBN}}^{2}}
=\Big(\frac{\gamma}{\sigma}\Big)^{2},
\label{eq:fixed-stats}
\end{equation}
\end{remark}



\subsection{Gradient direction}
\label{sec:direction}

Amplified gradient \emph{norms} alone do not establish memorization: a uniformly
rescaled step along the shared bulk direction would raise norms without making
the model any more outlier-specific. 

\begin{figure}[t]
    \centering
    \includegraphics[width=1.0\columnwidth]{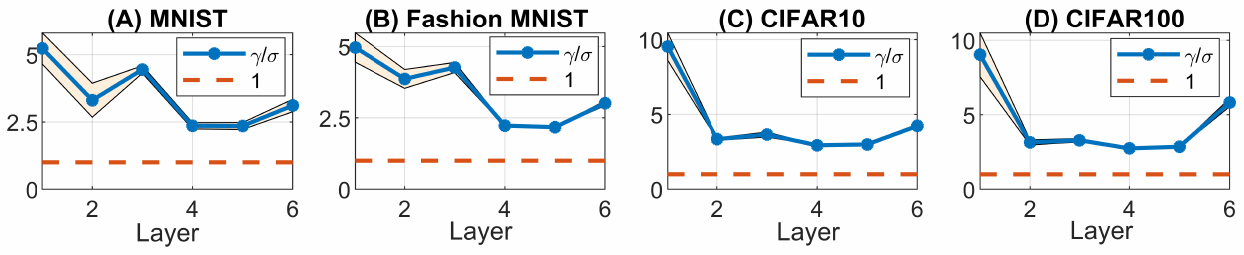}
    \caption{Ratio $\gamma/\sigma$ of BN layers on Lenet architecture with various datasets and $k=10\%$. }
    \label{fig:parameter_ratio1}
\end{figure}

\begin{proposition}[Direction reshaping]
\label{prop:direction}
In the memorization regime the exact batch gradient is
\begin{equation}
\begin{aligned}
\nabla_w L_{\mathrm{BN}}
&= \frac{a\gamma\rho_{\star}}{\sigma}
\Bigg[
\underbrace{\kappa\,x^{\star}}_{\text{outlier-specific}}
-\underbrace{\frac{1}{B}\sum_{j\neq\star}
(1+u_jt)\,x_j}_{\text{bulk back-reaction}}
\Bigg],\\[2pt]
\nabla_w L_{\mathrm{noBN}}
&= a\rho_{\star}\,x^{\star}.
\end{aligned}
\label{eq:direction}
\end{equation}
The no-BN gradient is a scalar multiple of $x^{\star}$; the BN gradient is not.
The cosine between them is
$\cos\theta=\kappa\|x^{\star}\|_2/\|X^{\top}p_{\star}\|_2\le1$, with equality
iff the bulk back-reaction vanishes.
\end{proposition}

The bulk back-reaction term is nonzero even though $\rho_j\approx0$ for
$j\neq\star$: BN routes the outlier's error signal onto \emph{other} samples'
input directions through $\mu$ and $\sigma$. Therefore, BN actively \emph{decorrelates} the outlier's update from the bulk-mean gradient direction, which is the geometric signature of memorization rather than of feature learning. 

\subsection{Why $\gamma/\sigma>1$ emerges: a differential-population argument}
\label{sec:gamma-sigma}

Corollary~\ref{cor:amplification} makes $\gamma/\sigma$ the sole model-dependent
driver of amplification. Treating $\gamma/\sigma>1$ as a bare empirical
regularity would leave the mechanism incomplete. We now argue that the
inequality is a \emph{consequence} of the same tail dynamics the amplification
describes, because $\gamma$ and $\sigma$ are driven by disjoint populations of
samples.

\begin{proposition}[$\gamma$ is tail-driven; the gradient is exact]
\label{prop:gamma-exact}
Since $u_i$ does not depend on $\gamma$, the gradient
\begin{equation}
\frac{\partial L}{\partial\gamma}\;=\;a\sum_{i}\rho_i u_i\;=\;aB\,\overline{\rho u}
\label{eq:gamma-grad}
\end{equation}
is \emph{exact}. The tail sample
contributes $|\partial\ell_{\star}/\partial\gamma|=\varphi(-m_{\star})|a||t|
=\Theta(|t|)$, whereas a well-classified typical sample with $u_i\approx0$ or
$\varphi(-m_i)\approx0$ contributes $\approx0$.
\end{proposition}

\begin{proposition}[$\sigma$ is bulk-driven]
\label{prop:sigma-bulk}
The tail sample's share of the batch variance is exactly
\begin{equation}
\frac{(h_{\star}-\mu)^{2}}{\sum_i (h_i-\mu)^{2}}\;=\;\frac{t^{2}}{B}\;\le\;\frac{B-1}{B},
\label{eq:variance-share}
\end{equation}
so for $t^{2}\ll B$ we have $\sigma=\sigma_{\mathrm{bulk}}(1+O(t^{2}/B))$: the
denominator of the amplification ratio is set by the typical-sample spread and
is insensitive to the tail. At $B=256$, a $5\sigma$ outlier accounts for $10\%$
of the variance and a $3\sigma$ outlier for $3.5\%$.
\end{proposition}

\noindent\textbf{The asymmetry.}
Propositions~\ref{prop:gamma-exact} and \ref{prop:sigma-bulk} identify a
population asymmetry: \emph{$\gamma$ tracks the tail while $\sigma$ tracks the
bulk}. Once atypical samples are present and not yet fit, they supply a
persistent $\Theta(|t|)$ drive on $\gamma$ that typical samples do not, while
contributing only $O(t^2/B)$ to $\sigma$, which additionally stabilizes as the
bulk features converge. The ratio $\gamma/\sigma$ therefore rises during
training, and does so \emph{in proportion to the memorization load}. This makes
$\gamma/\sigma>1$ a signature of the mechanism rather than an independent
assumption: absent a memorizable tail, $\overline{\rho u}\approx0$, $\gamma$ is
not driven upward, and the ratio remains near its initialization
$\gamma_0/\sigma_0$.

\noindent\textbf{Notation.} From here on $n$ denotes the training-step index
and $t$ remains the tail deviation of Definition~\ref{def:tail}.

\begin{theorem}[Self-reinforcing memorization loop]
\label{thm:loop}
Consider the coupled dynamics of $m_{n}=y^{\star}z_{n}$ and
$\gamma_{n}$ under gradient descent with step size $\eta$ on a batch
containing one tail sample with $|t|\gg1$ and $t^{2}<B-1$, with $\sigma$ and
$\kappa_{\mathrm{eff}}$ stationary as in (ii). Write
\begin{equation}
\alpha:=\frac{\eta\,a^{2}\kappa_{\mathrm{eff}}^{2}\|x^{\star}\|_2^{2}}{\sigma^{2}},
\qquad
\beta:=\eta\,|a|\,|t|,
\label{eq:alphabeta}
\end{equation}
so that Propositions~\ref{prop:exact-margin} and~\ref{prop:gamma-exact} give
$m_{n+1}=m_n+\alpha\gamma_n^{2}\varphi(-m_n)$ and
$\gamma_{n+1}=\gamma_n+\beta\,\varphi(-m_n)$. Then:
\begin{enumerate}\itemsep2pt
\item[(i)] $\Delta\gamma_{n}=\eta\,\varphi(-m_{n})\,|a|\,|t|>0$
whenever $y^{\star}$ and $a$ are aligned, and this expression is exact
(Proposition~\ref{prop:gamma-exact});
\item[(ii)] $\sigma$ is stationary to $O(t^{2}/B)$ under the tail's
contribution (Proposition~\ref{prop:sigma-bulk}), so
$\gamma_{n}\!\uparrow$ implies
$(\gamma_{n}/\sigma)^{2}\!\uparrow$, which by
Corollary~\ref{cor:amplification} increases $\Delta m$ at the next step, since
$\kappa_{\mathrm{eff}}$ does not depend on $\gamma$:
\begin{equation}
\text{large }|t|\;\xrightarrow{\text{(i)}}\;\gamma\uparrow
\;\xrightarrow{\text{Cor.~\ref{cor:amplification}}}\;(\gamma/\sigma)^{2}\uparrow
\;\xrightarrow{\text{Prop.~\ref{prop:exact-margin}}}\;\Delta m\uparrow;
\label{eq:loop}
\end{equation}
\item[(iii)] 
Both
$\Delta\gamma_n\to0$ and $\Delta m_n\to0$, yet $\gamma_n$ is strictly increasing and \textit{unbounded}: the trajectory obeys the exact invariant
\begin{equation}
\tfrac{\alpha}{3}\gamma_n^{3}-\beta m_n=\mathrm{const},
\;\;\text{i.e.}\;\;
\gamma_n=\Big(\gamma_0^{3}+\tfrac{3\beta}{\alpha}(m_n-m_0)\Big)^{1/3},
\label{eq:invariant}
\end{equation}
and since $m_n=\log n+O(\log\log n)$,
\begin{equation}
\gamma_n=\Theta\big((\log n)^{1/3}\big),
\quad
\frac{3\beta}{\alpha}=\frac{3\,|t|\,\sigma^{2}}
{|a|\,\kappa_{\mathrm{eff}}^{2}\,\|x^{\star}\|_2^{2}} .
\label{eq:gamma-rate}
\end{equation}
There is thus no finite steady state $\gamma_{\infty}$. There is, however, a
finite $N$ past which $\gamma_n/\sigma>1$ for all $n\ge N$, and $N=0$ whenever
$\sigma\le\gamma_0$.
\end{enumerate}
\end{theorem}

\begin{corollary}[Finite-horizon form]
\label{cor:finite-horizon}
Training runs for a finite number of steps $T$. Over that horizon
\begin{equation}
\gamma_T=\Big(\gamma_0^{3}+\tfrac{3\beta}{\alpha}\log T\Big)^{1/3}
+O\!\Big(\tfrac{\log\log T}{(\log T)^{2/3}}\Big),
\label{eq:finite-horizon}
\end{equation}
and the residual drive at the horizon is
$\Delta\gamma_T=\beta\varphi(-m_T)=\Theta\big(T^{-1}(\log T)^{-2/3}\big)$, i.e.\
negligible. The loop is therefore \emph{practically} quiescent by the end of
training even though it does not terminate asymptotically.
\end{corollary}



\begin{remark}[Gauge invariance]
\label{rem:gauge}
BN makes the network invariant to the scale of $w$, and $\sigma$ co-scales with
$\|w\|$; the \emph{absolute} value of $\gamma/\sigma$ is therefore partly a gauge
artifact. Our claims are stated in two gauge-consistent forms: (a) the
matched-initialization BN vs.\ no-BN comparison of
Corollary~\ref{cor:amplification}, in which the gauge cancels, and (b) the
\emph{monotone trend} of $\gamma/\sigma$ with corruption ratio $k$ within a fixed
architecture. We do not rely on the bare threshold $\gamma/\sigma=1$ as a
scale-free statement.
\end{remark}

\subsection{Multi-step memorization speedup}
\label{sec:multistep}

\begin{proposition}[Time-to-memorize]
\label{prop:multistep}
Under the discrete dynamics
$m_{n+1}=m_{n}+c\,\varphi(-m_{n})$ induced by
Proposition~\ref{prop:exact-margin} at fixed $\gamma$, with
$c=\eta s^{2}\kappa_{\mathrm{eff}}^{2}\|x^{\star}\|_2^{2}$ and
$c_{\mathrm{noBN}}=\eta a^{2}\|x^{\star}\|_2^{2}$, the sequence
$\{m_{n}\}$ is
strictly increasing and unbounded, and the continuous-time travel time from
$m_0$ to a target margin $M$ is
$T(c)=\frac1c[(M+e^{M})-(m_0+e^{m_0})]$. Hence
\begin{equation}
\frac{T_{\mathrm{noBN}}}{T_{\mathrm{BN}}}
=\frac{c_{\mathrm{BN}}}{c_{\mathrm{noBN}}}
=\Big(\frac{\gamma}{\sigma}\Big)^{2}\kappa_{\mathrm{eff}}^{2},
\label{eq:speedup}
\end{equation}
independently of $m_0$, $M$, $\|x^{\star}\|$ and $\eta$.
\end{proposition}



\section{Memorization mitigation}
\label{sec:mem_mitigation}

\begin{figure}[t]
    \centering
    \includegraphics[width=\linewidth]{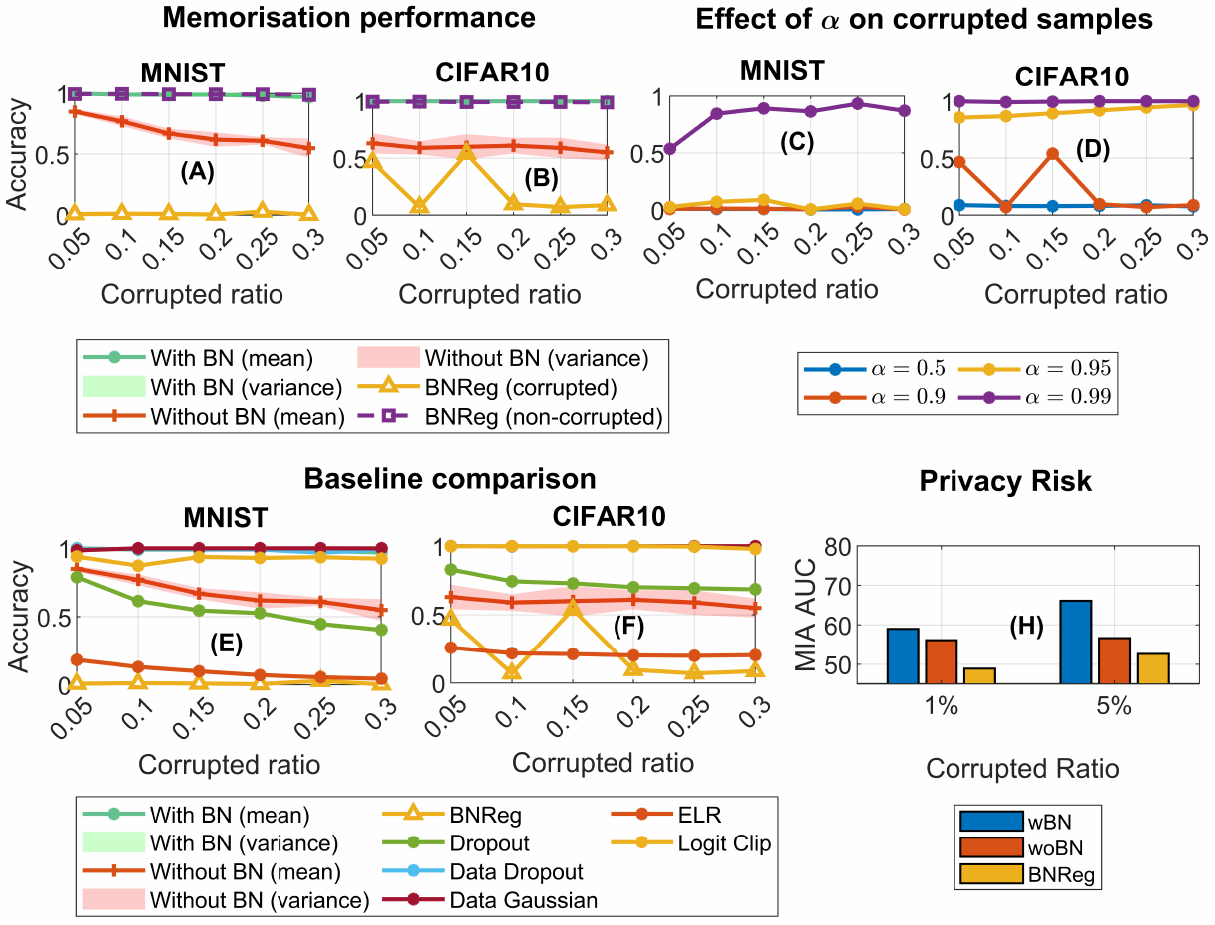}
    \caption{Memorization is reduced significantly when training with the BNReg algorithm (Figs. A, B). Figs. C, D ablate the regularization weight. Figs. E, F compare BNReg to existing works. Fig. H compares BNReg with baselines under MIA.}
    \label{fig:regularisation}
\end{figure}

Motivated by \ref{cor:amplification}, we propose $\gamma/\sigma$-aware approach named \textbf{BNReg}; we constrain this ratio during learning by introducing an additional $\gamma/\sigma$ regularization term in the loss function. Due to the classification penalty lying in the log space, we also consistently project the ratio regularization to the same log space. Formally, given $L$-layer model $f(\cdot)$ with weights $W$ and BN parameters $\gamma_i, \sigma_i$ at the layer $i$, we minimize the objective function: 
\begin{equation}
    J = \alpha \sum_{i=1}^{n} \text{CE}(f(x_i), y_i) + (1-\alpha)\sum_{j=1}^{L}\log\left((\gamma_j/\sigma_j)^2\right)
\end{equation} 
where $\alpha$ denotes a hyperparameter for balancing two terms. 

\noindent\textbf{Results and Ablation study.} Figure \ref{fig:regularisation} A and B confirm that BNReg (yellow lines) significantly reduces the memorization on the corrupted set compared to removing BN in the architecture from 100\% to mostly 0\% while maintaining the non-corrupted samples' accuracy (generalization) close to 100\%. We conduct an ablation study for the hyperparameter $\alpha$ that controls the preference weight between the predictive objective and the regularization term. Figures~\ref{fig:regularisation} C and D show that decreasing $\alpha$ reduces outlier memorization, indicating a memorization–generalization trade-off.

\noindent\textbf{Comparison to existing works.} BNReg consistently outperforms other memorization mitigation methods, which target diverse aspects: (1) data space (Data Dropout, Data Gaussian), (2) model space (Dropout, LogitClip~\cite{wei2023mitigating}), and (3) loss space (ERL~\cite{liu2020early}); c.f. Figure ~\ref{fig:regularisation} E, F.

\noindent\textbf{Privacy evaluation.} To evaluate the actual privacy mitigation achieved by BNReg, we conduct MIA on BN-Lenet trained with BNReg. As shown in Figure~\ref{fig:regularisation} H, incorporating BNReg during training reduces MIA leakage more effectively than simply removing BN layers, demonstrating the mitigation benefit of our regularized algorithm.
\section{Related work}
\label{sec:related works}

\textbf{BN and Generalization.} Literature showed that BN accelerates learning convergence \cite{ioffe2015batch}, stabilizes the learning process \cite{bjorck2018understanding}, and smoothes the loss landscape \cite{santurkar2018does}. Other works \cite{lyu2022understanding, hochreiter1997flat} proved theoretically that BN reduces loss sharpness, guiding the model toward flatter minima that generalize better \cite{hochreiter1997flat,keskar2016large,neyshabur2017exploring}. 
These works primarily focus on generalization, rather than its role in memorization.


\noindent\textbf{BN and Security.} Existing work found that BN decrease adversarial robustness \cite{galloway2019batch,benz2021batch}. BN fails to capture the statistical characteristics of the mixture of clean data and perturbed data, thereby degrading generalization under adversarial training \cite{Xie2020Intriguing}.
Other work \cite{wang2022removing, zhang2022achieving} used normalization-free architectures to enhance robustness. Although these studies reveal a potential link between BN and various attacking mechanisms, the connection between BN and model memorization remains largely unexplored.

\noindent\textbf{Memorization and Architectures.} \cite{mo2021quantifying,baldock2021deep,maini2023can} show that memorization can be associated with specific parts of a model.  \cite{stephenson2021geometry} shows that memorization is localized at the last layers of a model. \cite{wongso2023using} exploits information theory and shows a similar finding that all layers start to memorize label noise data, but earlier layers quickly become stable. \cite{hintersdorf2024finding}, \cite{zhang2024does} and ~\cite{singhal2026impact} localize neurons responsible for memorization in diffusion models, CNNs and Transformers. Our work differs in that rather than concentrating on the effects of neurons on memorization, we study the interplay between BN layers and memorization.
\section{Conclusion}
\label{sec:conclusion}

We investigate the privacy implications of BN. Extensive empirical evaluations across diverse architectures and datasets reveal that incorporating BN introduces notable privacy vulnerabilities. We further provide theoretical analysis to support these observations. Our results suggest practitioners should carefully consider the privacy implications of BN in practical applications.

{
    \small
    \bibliographystyle{splncs04}
    \bibliography{main}

@String(CVPR  = {IEEE Conf. Comput. Vis. Pattern Recog.})

@String(ECCV  = {Eur. Conf. Comput. Vis.})

@String(NeurIPS = {Adv. Neural Inform. Process. Syst.})

@String(ICML  = {Int. Conf. Mach. Learn.})

@String(ICLR  = {Int. Conf. Learn. Represent.})

@String(CVPR  = {CVPR})

@String(ECCV  = {ECCV})

@String(NeurIPS = {NeurIPS})

@String(ICML  = {ICML})

@String(ICLR  = {ICLR})

@article{deng2012mnist,
  title={The mnist database of handwritten digit images for machine learning research},
  author={Deng, Li},
  journal={SPM},
  volume={29},
  number={6},
  pages={141--142},
  year={2012}
}

@article{xiao2017fashion,
  title={Fashion-mnist: a novel image dataset for benchmarking machine learning algorithms},
  author={Xiao, Han and Rasul, Kashif and Vollgraf, Roland},
  journal={arXiv preprint arXiv:1708.07747},
  year={2017}
}

@article{krizhevsky2009learning,
  title={Learning multiple layers of features from tiny images},
  author={Krizhevsky, Alex and Hinton, Geoffrey and others},
  year={2009},
  publisher={Toronto, ON, Canada}
}

@article{lecun2002gradient,
  title={Gradient-based learning applied to document recognition},
  author={LeCun, Yann and Bottou, L{\'e}on and Bengio, Yoshua and Haffner, Patrick},
  journal={Proc. IEEE},
  volume={86},
  number={11},
  pages={2278--2324},
  year={2002}
}

@inproceedings{he2016deep,
  title={Deep residual learning for image recognition},
  author={He, Kaiming and Zhang, Xiangyu and Ren, Shaoqing and Sun, Jian},
  booktitle={CVPR},
  pages={770--778},
  year={2016}
}

@inproceedings{wang2022removing,
  title={Removing batch normalization boosts adversarial training},
  author={Wang, Haotao and Zhang, Aston and Zheng, Shuai and Shi, Xingjian and Li, Mu and Wang, Zhangyang},
  booktitle={ICML},
  pages={23433--23445},
  year={2022},
  organization={PMLR}
}

@inproceedings{zhang2022achieving,
  title={Achieving Both Model Accuracy and Robustness by Adversarial Training with Batch Norm Shaping},
  author={Zhang, Brian and Ma, Shiqing},
  booktitle={ICTAI},
  pages={591--598},
  year={2022},
  organization={IEEE}
}

@inproceedings{koh2017understanding,
  title={Understanding black-box predictions via influence functions},
  author={Koh, Pang Wei and Liang, Percy},
  booktitle={ICML},
  pages={1885--1894},
  year={2017},
  organization={PMLR}
}

@inproceedings{feldman20,
 author = {Feldman, Vitaly and Zhang, Chiyuan},
 booktitle = {NeurIPS},
 editor = {H. Larochelle and M. Ranzato and R. Hadsell and M.F. Balcan and H. Lin},
 pages = {2881--2891},
 title = {What Neural Networks Memorize and Why: Discovering the Long Tail via Influence Estimation},
 volume = {33},
 year = {2020}
}

@inproceedings{carlini2022membership,
  title={Membership inference attacks from first principles},
  author={Carlini, Nicholas and Chien, Steve and Nasr, Milad and Song, Shuang and Terzis, Andreas and Tramer, Florian},
  booktitle={SP},
  pages={1897--1914},
  year={2022}
}

@inproceedings{huang2017densely,
  title={Densely connected convolutional networks},
  author={Huang, Gao and Liu, Zhuang and Van Der Maaten, Laurens and Weinberger, Kilian Q},
  booktitle={CVPR},
  pages={4700--4708},
  year={2017}
}

@inproceedings{Xie2020Intriguing,
title={Intriguing Properties of Adversarial Training at Scale},
author={Cihang Xie and Alan Yuille},
booktitle={ICLR},
year={2020}
}

@inproceedings{ioffe2015batch,
  title={Batch normalization: Accelerating deep network training by reducing internal covariate shift},
  author={Ioffe, Sergey and Szegedy, Christian},
  booktitle={ICLR},
  pages={448--456},
  year={2015}
}

@article{santurkar2018does,
  title={How does batch normalization help optimization?},
  author={Santurkar, Shibani and Tsipras, Dimitris and Ilyas, Andrew and Madry, Aleksander},
  journal={NeurIPS},
  volume={31},
  year={2018}
}

@INPROCEEDINGS{9378171,
  author={Yao, Zhewei and Gholami, Amir and Keutzer, Kurt and Mahoney, Michael W.},
  booktitle={Big Data}, 
  title={PyHessian: Neural Networks Through the Lens of the Hessian}, 
  year={2020},
  volume={},
  number={},
  pages={581-590}
}

@article{Kong2023BatchNormVulnerability,
  author    = {Kong, Fan and Liu, Fei and Xu, Kai and others},
  title     = {Why does batch normalization induce the model vulnerability on adversarial images?},
  journal   = {WWW},
  volume    = {26},
  pages     = {1073--1091},
  year      = {2023}
}

@inproceedings{benz2021batch,
  title={Batch normalization increases adversarial vulnerability and decreases adversarial transferability: A non-robust feature perspective},
  author={Benz, Philipp and Zhang, Chaoning and Kweon, In So},
  booktitle={CVPR},
  pages={7818--7827},
  year={2021}
}

@article{bjorck2018understanding,
  title={Understanding batch normalization},
  author={Bjorck, Nils and Gomes, Carla P and Selman, Bart and Weinberger, Kilian Q},
  journal={NeurIPS},
  volume={31},
  year={2018}
}

@article{lyu2022understanding,
  title={Understanding the generalization benefit of normalization layers: Sharpness reduction},
  author={Lyu, Kaifeng and Li, Zhiyuan and Arora, Sanjeev},
  journal={NeurIPS},
  volume={35},
  pages={34689--34708},
  year={2022}
}

@article{galloway2019batch,
  title={Batch normalization is a cause of adversarial vulnerability},
  author={Galloway, Angus and Golubeva, Anna and Tanay, Thomas and Moussa, Medhat and Taylor, Graham W},
  journal={arXiv preprint arXiv:1905.02161},
  year={2019}
}

@article{carlini2023extracting,
  title={Extracting training data from diffusion models},
  author={Carlini, Nicholas and Hayes, Jamie and Nasr, Milad and Jagielski, Matthew and Choquette-Choo, Christopher A and Balle, Borja and Tramer, Florian and Wallace, Eric and Song, Dawn and others},
  journal={USENIX},
  year={2023}
}

@inproceedings{carlini2019secret,
  title={The secret sharer: Evaluating and testing unintended memorization in neural networks},
  author={Carlini, Nicholas and Liu, Chang and Kos, Jernej and Erlingsson, {\'U}lfar and Song, Dawn},
  booktitle={USENIX},
  pages={267--284},
  year={2019}
}

@inproceedings{shokri2017membership,
  title={Membership inference attacks against machine learning models},
  author={Shokri, Reza and Stronati, Marco and Song, Congzheng and Shmatikov, Vitaly},
  booktitle={SP},
  pages={3--18},
  year={2017},
  organization={IEEE}
}

@inproceedings{song2021systematic,
  title={Systematic evaluation of privacy risks of machine learning models},
  author={Song, Congzheng and Shokri, Reza},
  booktitle={USENIX},
  pages={2615--2632},
  year={2021}
}

@article{mo2021quantifying,
  title={Quantifying information leakage from gradients},
  author={Mo, Fan and Borovykh, Anastasia and Malekzadeh, Mohammad and Haddadi, Hamed and Demetriou, Soteris},
  journal={arXiv preprint arXiv:2105.13929},
  year={2021},
  publisher={ArXiv}
}

@article{baldock2021deep,
  title={Deep learning through the lens of example difficulty},
  author={Baldock, Robert and Maennel, Hartmut and Neyshabur, Behnam},
  journal={NeurIPS},
  volume={34},
  pages={10876--10889},
  year={2021}
}

@article{maini2023can,
  title={Can neural network memorization be localized?},
  author={Maini, Pratyush and Mozer, Michael C and Sedghi, Hanie and Lipton, Zachary C and Kolter, J Zico and Zhang, Chiyuan},
  journal={arXiv preprint arXiv:2307.09542},
  year={2023}
}

@article{stephenson2021geometry,
  title={On the geometry of generalization and memorization in deep neural networks},
  author={Stephenson, Cory and Padhy, Suchismita and Ganesh, Abhinav and Hui, Yue and Tang, Hanlin and Chung, SueYeon},
  journal={arXiv preprint arXiv:2105.14602},
  year={2021}
}

@inproceedings{wongso2023using,
  title={Using sliced mutual information to study memorization and generalization in deep neural networks},
  author={Wongso, Shelvia and Ghosh, Rohan and Motani, Mehul},
  booktitle={AISTATS},
  pages={11608--11629},
  year={2023},
  organization={PMLR}
}

@article{hintersdorf2024finding,
  title={Finding nemo: Localizing neurons responsible for memorization in diffusion models},
  author={Hintersdorf, Dominik and Struppek, Lukas and Kersting, Kristian and Dziedzic, Adam and Boenisch, Franziska},
  journal={NeurIPS},
  volume={37},
  pages={88236--88278},
  year={2024}
}

@inproceedings{10.5555/3600270.3601234,
author = {Carlini, Nicholas and Terzis, Andreas and Jagielski, Matthew and Tramer, Florian and Papernot, Nicolas and Zhang, Chiyuan},
title = {The privacy onion effect: memorization is relative},
year = {2022},
isbn = {9781713871088},
address = {Red Hook, NY, USA},
booktitle = {NeurIPS},
articleno = {964},
numpages = {14},
series = {NIPS '22}
}

@article{hochreiter1997flat,
  title={Flat minima},
  author={Hochreiter, Sepp and Schmidhuber, J{\"u}rgen},
  journal={Neural Comput.},
  volume={9},
  number={1},
  pages={1--42},
  year={1997}
}

@article{keskar2016large,
  title={On large-batch training for deep learning: Generalization gap and sharp minima},
  author={Keskar, Nitish Shirish and Mudigere, Dheevatsa and Nocedal, Jorge and Smelyanskiy, Mikhail and Tang, Ping Tak Peter},
  journal={arXiv preprint arXiv:1609.04836},
  year={2016}
}

@article{neyshabur2017exploring,
  title={Exploring generalization in deep learning},
  author={Neyshabur, Behnam and Bhojanapalli, Srinadh and McAllester, David and Srebro, Nati},
  journal={NeurIPS},
  volume={30},
  year={2017}
}

@article{dosovitskiy2020image,
  title={An image is worth 16x16 words: Transformers for image recognition at scale},
  author={Dosovitskiy, Alexey and Beyer, Lucas and Kolesnikov, Alexander and Weissenborn, Dirk and Zhai, Xiaohua and Unterthiner, Thomas and Dehghani, Mostafa and Minderer, Matthias and Heigold, Georg and Gelly, Sylvain and others},
  journal={arXiv preprint arXiv:2010.11929},
  year={2020}
}

@article{zhang2016understanding,
  title={Understanding deep learning requires rethinking generalization},
  author={Zhang, Chiyuan and Bengio, Samy and Hardt, Moritz and Recht, Benjamin and Vinyals, Oriol},
  journal={arXiv preprint arXiv:1611.03530},
  year={2016}
}

@article{carlini2022privacy,
  title={The privacy onion effect: Memorization is relative},
  author={Carlini, Nicholas and Jagielski, Matthew and Zhang, Chiyuan and Papernot, Nicolas and Terzis, Andreas and Tramer, Florian},
  journal={NeurIPS},
  volume={35},
  pages={13263--13276},
  year={2022}
}

@article{tirumala2022memorization,
  title={Memorization without overfitting: Analyzing the training dynamics of large language models},
  author={Tirumala, Kushal and Markosyan, Aram and Zettlemoyer, Luke and Aghajanyan, Armen},
  journal={NeurIPS},
  volume={35},
  pages={38274--38290},
  year={2022}
}

@article{zhang2023counterfactual,
  title={Counterfactual memorization in neural language models},
  author={Zhang, Chiyuan and Ippolito, Daphne and Lee, Katherine and Jagielski, Matthew and Tram{\`e}r, Florian and Carlini, Nicholas},
  journal={NeurIPS},
  volume={36},
  pages={39321--39362},
  year={2023}
}

@inproceedings{netzer2011reading,
  title={Reading digits in natural images with unsupervised feature learning},
  author={Netzer, Yuval and Wang, Tao and Coates, Adam and Bissacco, Alessandro and Wu, Baolin and Ng, Andrew Y and others},
  booktitle={NIPS workshop on deep learning and unsupervised feature learning},
  volume={2011},
  number={2},
  pages={4},
  year={2011}
}

@article{huynh2022vision,
  title={Vision transformers in 2022: An update on tiny imagenet},
  author={Huynh, Ethan},
  journal={arXiv preprint arXiv:2205.10660},
  year={2022}
}

@article{ulyanov2016instance,
  title={Instance normalization: The missing ingredient for fast stylization},
  author={Ulyanov, Dmitry and Vedaldi, Andrea and Lempitsky, Victor},
  journal={arXiv preprint arXiv:1607.08022},
  year={2016}
}

@inproceedings{wu2018group,
  title={Group normalization},
  author={Wu, Yuxin and He, Kaiming},
  booktitle={ECCV},
  pages={3--19},
  year={2018}
}

@inproceedings{shen2020powernorm,
  title={Powernorm: Rethinking batch normalization in transformers},
  author={Shen, Sheng and Yao, Zhewei and Gholami, Amir and Mahoney, Michael and Keutzer, Kurt},
  booktitle={ICML},
  pages={8741--8751},
  year={2020}
}

@article{hoffer2018norm,
  title={Norm matters: efficient and accurate normalization schemes in deep networks},
  author={Hoffer, Elad and Banner, Ron and Golan, Itay and Soudry, Daniel},
  journal={NeurIPS},
  volume={31},
  year={2018}
}

@misc{ba2016layernormalization,
      title={Layer Normalization}, 
      author={Jimmy Lei Ba and Jamie Ryan Kiros and Geoffrey E. Hinton},
      year={2016},
      archivePrefix={arXiv},
      url={https://arxiv.org/abs/1607.06450}, 
}

@article{liu2020early,
  title={Early-learning regularization prevents memorization of noisy labels},
  author={Liu, Sheng and Niles-Weed, Jonathan and Razavian, Narges and Fernandez-Granda, Carlos},
  journal={NeurIPS},
  volume={33},
  pages={20331--20342},
  year={2020}
}

@inproceedings{wei2023mitigating,
  title={Mitigating memorization of noisy labels by clipping the model prediction},
  author={Wei, Hongxin and Zhuang, Huiping and Xie, Renchunzi and Feng, Lei and Niu, Gang and An, Bo and Li, Yixuan},
  booktitle={ICML},
  pages={36868--36886},
  year={2023}
}

@article{singhal2026impact,
  title={Impact of Layer Norm on Memorization and Generalization in Transformers},
  author={Singhal, Rishi and Kim, Jung-Eun},
  journal={NeurIPS},
  volume={38},
  pages={77239--77291},
  year={2026}
}

@inproceedings{zhang2024does,
  title={How does a deep learning model architecture impact its privacy? a comprehensive study of privacy attacks on $\{$CNNs$\}$ and transformers},
  author={Zhang, Guangsheng and Liu, Bo and Tian, Huan and Zhu, Tianqing and Ding, Ming and Zhou, Wanlei},
  booktitle={USENIX},
  pages={6795--6812},
  year={2024}
}

@inproceedings{ye2024leave,
  title={Leave-one-out distinguishability in machine learning},
  author={Ye, Jiayuan and Borovykh, Anastasia and Hayou, Soufiane and Shokri, Reza},
  booktitle={ICLR},
  volume={2024},
  pages={40267--40303},
  year={2024}
}

@article{hoyos2023representation,
  title={The representation jensen-shannon divergence},
  author={Hoyos-Osorio, Jhoan K and Sanchez-Giraldo, Luis G},
  journal={arXiv preprint arXiv:2305.16446},
  year={2023}
}
}

\clearpage

\appendix

\tableofcontents
\setcounter{equation}{12}
\setcounter{figure}{8}

\begin{center}
    \textbf{ \Large Appendices}
\end{center}
\label{sec:append}

\noindent In the appendices, we extend our theory framework to understand how BN accelerates memorization with proofs and theoretical insights in Section \ref{app:theory}. We show the pseudocode for synthesizing corrupted datasets in Section \ref{sec:ap:pseudos}. We further examine characteristics of models on memorizing corrupted samples in Section \ref{sec:ap:detailed_analysis}. Experimental setup details, including hyperparameters, MIA settings, and evaluation metrics, are presented in Section \ref{sec:ap:implementation_details}. In Section \ref{sec:ap:diff_hyperparams}, we investigate the effects of different parameters, including learning rates, optimizers, normalization techniques, and model architectures, on memorization. Then, we discuss the trade-off between memorization and generalization in Section \ref{sec:ap:tradeoff}. Finally, we study per-sample influence, another proxy to estimate memorization of BN models and no-BN models in Section \ref{sec:ap:per_sample}.

\section{Mechanistic Insights}
\label{app:theory}

Our empirical results consistently show that BN amplifies the memorization of
outlier samples across architectures and datasets. This appendix gives the full
derivations. The analysis proceeds in four stages:

\begin{enumerate}\itemsep2pt
\item \textbf{Exact backward pass} (\S\ref{app:backward}). We derive the BN
backward as an orthogonal projection and bound the tail sample's leverage. This
is the primary technical result; the fixed-statistics model of prior work is
recovered as its large-batch limit.
\item \textbf{Per-step amplification} (\S\ref{app:perstep}). BN amplifies the
per-step margin growth on a tail sample by
$(\gamma/\sigma)^{2}\kappa_{\mathrm{eff}}^{2}$, with
$\kappa_{\mathrm{eff}}\in[0,1]$ a batch-size-controlled attenuation, and the
growth is nonnegative for every batch composition.
\item \textbf{Why $\gamma/\sigma>1$} (\S\ref{app:gamma-sigma}). $\gamma$ is
driven by the tail and $\sigma$ by the bulk; this population asymmetry, not a
contingent empirical fact, is what produces the ratio.
\item \textbf{Multi-step speedup} (\S\ref{app:multistep}). The per-step advantage
compounds into a $(\sigma/\gamma)^{2}\kappa_{\mathrm{eff}}^{-2}$ reduction in
time-to-memorize.
\end{enumerate}

We use the notation of Section~\ref{sec:theory} throughout:
$h_i=w^{\top}x_i+b$, $\mu=\tfrac1B\sum_i h_i$,
$\sigma^{2}=\tfrac1B\sum_i(h_i-\mu)^{2}$, $u_i=(h_i-\mu)/\sigma$,
$\hat h_i=\gamma u_i+\beta$, $z_i=a\hat h_i+c$,
$\varphi(v)=1/(1+e^{-v})$, $m_i=y_iz_i$,
$\rho_i=\partial\ell_i/\partial z_i=-y_i\varphi(-m_i)$, $L=\sum_i\ell_i$,
$X\in\mathbb{R}^{B\times d}$ with rows $x_i^{\top}$, and $K=XX^{\top}$. We index
the tail sample by $\star$ and write $t=u_{\star}$.

\subsection{The exact BN backward pass}
\label{app:backward}

\subsubsection{Proof of Lemma~\ref{lem:projection}}

\begin{proof}
\emph{Step 1: the two constraints.} By definition of $\mu$,
$\sum_i(h_i-\mu)=0$, hence $\mathbf{1}^{\top}u=\sum_i u_i=0$. By definition of
$\sigma^{2}$,
$\sum_i u_i^{2}=\sigma^{-2}\sum_i(h_i-\mu)^{2}=\sigma^{-2}\cdot B\sigma^{2}=B$,
so $\|u\|_2^{2}=B$.\footnote{With the numerical stabilizer, $u_i=(h_i-\mu)/
\sqrt{\sigma^{2}+\epsilon}$ and $\|u\|_2^{2}=B\sigma^{2}/(\sigma^{2}+\epsilon)\le B$,
so every inequality below holds a fortiori.}

\emph{Step 2: differentiating the normalization.} Fix $j$ and compute
$\partial u_i/\partial h_j$. From $u_i=(h_i-\mu)/\sigma$,
\[
\frac{\partial u_i}{\partial h_j}
=\frac{\delta_{ij}-\partial\mu/\partial h_j}{\sigma}
-\frac{h_i-\mu}{\sigma^{2}}\frac{\partial\sigma}{\partial h_j}
=\frac{\delta_{ij}-\tfrac1B}{\sigma}-\frac{u_i}{\sigma}\frac{\partial\sigma}{\partial h_j}.
\]
Differentiating $\sigma^{2}=\tfrac1B\sum_k(h_k-\mu)^{2}$ gives
$2\sigma\,\partial\sigma/\partial h_j=\tfrac2B\sum_k(h_k-\mu)(\delta_{kj}-\tfrac1B)
=\tfrac2B(h_j-\mu)$, using $\sum_k(h_k-\mu)=0$. Hence
$\partial\sigma/\partial h_j=u_j/B$ and
\begin{equation}
\frac{\partial u_i}{\partial h_j}
=\frac1\sigma\Big(\delta_{ij}-\frac1B-\frac{u_iu_j}{B}\Big)
=\frac1\sigma P_{ij}.
\label{eq:jacobian}
\end{equation}

\emph{Step 3: $P$ is an orthogonal projector.} Write
$q_1=\mathbf{1}/\sqrt{B}$ and $q_2=u/\sqrt{B}$. By Step~1,
$\|q_1\|_2=\|q_2\|_2=1$ and $q_1^{\top}q_2=\mathbf{1}^{\top}u/B=0$, so
$\{q_1,q_2\}$ is orthonormal and
$P=I-q_1q_1^{\top}-q_2q_2^{\top}$ is the orthogonal projector onto
$\operatorname{span}\{q_1,q_2\}^{\perp}$. Consequently $P=P^{\top}=P^{2}$ and
$\|P\|_2=1$.

\emph{Step 4: chain rule.} Since $z_i=a\gamma u_i+(a\beta+c)$ we have
$\partial L/\partial u_i=a\gamma\rho_i$, and by \eqref{eq:jacobian},
\[
\frac{\partial L}{\partial h_j}
=\sum_i\frac{\partial L}{\partial u_i}\frac{\partial u_i}{\partial h_j}
=\frac{a\gamma}{\sigma}\sum_i \rho_i P_{ij}
=\frac{a\gamma}{\sigma}\,(P\rho)_j,
\]
using $P=P^{\top}$. Expanding the two rank-one terms yields
\eqref{eq:exact-backward-components}. Finally, \eqref{eq:jacobian} states that
$\partial u/\partial h=\tfrac1\sigma P$, which is the same symmetric operator
appearing in the backward pass.
\end{proof}

\begin{remark}[What the fixed-statistics model discards]
Retaining only $\delta_{ij}$ in \eqref{eq:jacobian} gives
$\partial L/\partial h_j\approx\tfrac{a\gamma}{\sigma}\rho_j$, i.e.\ replaces the
projector $P$ by the identity. The discarded terms are exactly the
mean-subtraction channel ($-\tfrac1B\mathbf{1}\mathbf{1}^{\top}$) and the
variance-coupling channel ($-\tfrac1B uu^{\top}$). Both are rank one, which is
why their total effect is controlled by a scalar, $\kappa$.
\end{remark}

\subsubsection{Proof of Lemma~\ref{lem:leverage}}

\begin{proof}
By Step~1 above, $\sum_i u_i=0$ and $\sum_i u_i^{2}=B$. Writing $t=u_{\star}$,
the first constraint gives $t=-\sum_{j\neq\star}u_j$, and Cauchy--Schwarz gives
\[
t^{2}=\Big(\sum_{j\neq\star}u_j\Big)^{2}\le(B-1)\sum_{j\neq\star}u_j^{2}
=(B-1)\big(B-t^{2}\big).
\]
Rearranging, $t^{2}\big(1+(B-1)\big)\le B(B-1)$, i.e.\ $t^{2}\le B-1$.
Substituting into \eqref{eq:kappa},
\[
\kappa=1-\frac{1+t^{2}}{B}\ \ge\ 1-\frac{1+(B-1)}{B}=0,
\]
and $\kappa<1$ since $t^{2}\ge0$ implies $\kappa\le1-1/B$. Equality
$t^{2}=B-1$ holds iff Cauchy--Schwarz is tight, i.e.\ all $u_j$, $j\neq\star$,
are equal, giving $u_j=-t/(B-1)$; this is the single-spike batch. Consistency
check: $\kappa=P_{\star\star}$ is a diagonal entry of an orthogonal projector, so
$\kappa=e_{\star}^{\top}P^{2}e_{\star}=\|Pe_{\star}\|_2^{2}\in[0,1]$
independently of the above. For fixed $t$, $\kappa=1-(1+t^{2})/B\to1$ as
$B\to\infty$.
\end{proof}

\begin{remark}[Interpretation]
$\kappa$ is the leverage complement of the tail sample: $1-\kappa=(1+t^{2})/B$ is
the fraction of the normalization directions the sample itself occupies, split
into $1/B$ for the mean and $t^{2}/B$ for the variance. The second term is
precisely the "the outlier inflates $\sigma$ and moderates its own step"
effect, now quantified rather than asserted.
\end{remark}

\subsection{Per-step amplification}
\label{app:perstep}

\subsubsection{Proof of Proposition~\ref{prop:exact-margin}}

\begin{proof}
\emph{Step 1: the weight gradient.} By the chain rule and
Lemma~\ref{lem:projection}, with $\partial h_j/\partial w=x_j$,
\begin{equation}
g:=\nabla_w L=\sum_j\frac{\partial L}{\partial h_j}x_j
=\frac{a\gamma}{\sigma}X^{\top}P\rho .
\label{eq:wgrad}
\end{equation}
The gradient step is $\Delta w=-\eta g$.

\emph{Step 2: propagate to pre-activations.} To first order in $\eta$,
$\delta h=X\Delta w=-\eta\frac{a\gamma}{\sigma}XX^{\top}P\rho
=-\eta\frac{a\gamma}{\sigma}KP\rho$.

\emph{Step 3: propagate through the normalization, exactly.} The batch statistics
are re-evaluated after the update. By \eqref{eq:jacobian}, to first order,
\[
\delta u=\frac{\partial u}{\partial h}\,\delta h=\frac1\sigma P\,\delta h
=-\eta\frac{a\gamma}{\sigma^{2}}\,PKP\rho .
\]
Note that the \emph{same} projector appears here as in Step~1; this is the
symmetry noted in Lemma~\ref{lem:projection}. Hence
\[
\delta z=a\gamma\,\delta u=-\eta\frac{(a\gamma)^{2}}{\sigma^{2}}\,PKP\rho .
\]

\emph{Step 4: isolate the tail sample.} In the memorization regime the bulk is
well classified, $\rho_j\approx0$ for $j\neq\star$, so $\rho=\rho_{\star}e_{\star}$
and
\[
\delta z_{\star}=-\eta\frac{(a\gamma)^{2}}{\sigma^{2}}\rho_{\star}(PKP)_{\star\star}.
\]
Multiplying by $y_{\star}$ and using
$-y_{\star}\rho_{\star}=y_{\star}^{2}\varphi(-m_{\star})=\varphi(-m_{\star})$,
\[
\Delta m_{\star}=y_{\star}\,\delta z_{\star}
=\eta\,\varphi(-m_{\star})\Big(\frac{a\gamma}{\sigma}\Big)^{2}(PKP)_{\star\star}.
\]

\emph{Step 5: nonnegativity and the explicit form.} Writing $p_{\star}=Pe_{\star}$
and using $P=P^{\top}=P^{2}$,
\[
(PKP)_{\star\star}=e_{\star}^{\top}PXX^{\top}Pe_{\star}
=\|X^{\top}p_{\star}\|_2^{2}\ \ge\ 0 .
\]
The entries of $p_{\star}$ follow from $P_{ij}=\delta_{ij}-\tfrac1B-\tfrac{u_iu_j}{B}$
evaluated at $j=\star$: $(p_{\star})_{\star}=1-\tfrac1B-\tfrac{t^{2}}{B}=\kappa$
and $(p_{\star})_j=-\tfrac{1+u_jt}{B}$ for $j\neq\star$. Since $PKP\succeq0$ for
any Gram matrix $K\succeq0$, $\Delta m_{\star}\ge0$ regardless of batch
composition.
\end{proof}

\begin{remark}[Backward approximation vs.\ evaluation convention]
Step~3 differentiates the normalization when re-evaluating the margin, so the
result is first order in $\eta$ but makes \emph{no} fixed-statistics assumption
in either the forward or backward direction. This is the substantive difference
from the earlier analysis, which froze $(\mu,\sigma)$ in both.
\end{remark}

\subsubsection{Proof of Corollary~\ref{cor:amplification}}

\begin{proof}
For the no-BN model, $z_i=ah_i+c$ and there is no normalization, so
$\nabla_w\ell_{\star}=a\rho_{\star}x^{\star}$, $\delta h_{\star}
=-\eta a\rho_{\star}\|x^{\star}\|_2^{2}$ and
$\Delta m_{\mathrm{noBN}}=\eta\varphi(-m_{\star})a^{2}\|x^{\star}\|_2^{2}$,
which is \eqref{eq:exact-margin} with $P$ replaced by $I$ and $\gamma/\sigma$ by
$1$. At a matched initialization both models share $x^{\star}$, $m_{\star}$ and
hence $\varphi(-m_{\star})$; dividing, the factors $\eta$ and
$\varphi(-m_{\star})$ cancel and
\[
\mathcal{A}=\Big(\frac{\gamma}{\sigma}\Big)^{2}
\frac{\|X^{\top}p_{\star}\|_2^{2}}{\|x^{\star}\|_2^{2}}
=\Big(\frac{\gamma}{\sigma}\Big)^{2}\kappa_{\mathrm{eff}}^{2}.
\]
For the bound under bulk orthogonality, expand
\[
\|X^{\top}p_{\star}\|_2^{2}
=\kappa^{2}\|x^{\star}\|_2^{2}
+2\kappa\!\!\sum_{j\neq\star}\!(p_{\star})_j\langle x_j,x^{\star}\rangle
+\Big\|\sum_{j\neq\star}(p_{\star})_jx_j\Big\|_2^{2}.
\]
If $\langle x_j,x^{\star}\rangle\approx0$ for $j\neq\star$, the cross term
vanishes and the remaining term is nonnegative, giving
$\kappa_{\mathrm{eff}}^{2}\ge\kappa^{2}$; if in addition the bulk term is
negligible relative to $\kappa^{2}\|x^{\star}\|_2^{2}$, then
$\kappa_{\mathrm{eff}}^{2}\to\kappa^{2}$. A gauge-free upper bound also holds in
general: $\|X^{\top}p_{\star}\|_2^{2}\le\|X\|_2^{2}\|p_{\star}\|_2^{2}
=\|X\|_2^{2}\,\kappa$, using $\|p_{\star}\|_2^{2}=e_{\star}^{\top}P^{2}e_{\star}
=\kappa$.
\end{proof}

\subsubsection{Proof of Corollary~\ref{cor:fixed-stats}}

\begin{proof}
Fix $t$ and $\|x^{\star}\|_2$ and let $B\to\infty$. By \eqref{eq:kappa},
$\kappa=1-(1+t^{2})/B\to1$. The off-diagonal entries of $p_{\star}$ satisfy
$|(p_{\star})_j|=|1+u_jt|/B=O(1/B)$, so with bounded $\|x_j\|_2$ the bulk
back-reaction $\big\|\sum_{j\neq\star}(p_{\star})_jx_j\big\|_2=O(B^{-1/2})$ under
the standard assumption that the $u_j$ are $O(1)$ and the $x_j$ are not
systematically aligned. Hence $\kappa_{\mathrm{eff}}\to1$ and
\eqref{eq:exact-margin} reduces to
$\Delta m_{\star}=\eta\varphi(-m_{\star})s^{2}\|x^{\star}\|_2^{2}$ with
$s=a\gamma/\sigma$, and $\mathcal{A}\to(\gamma/\sigma)^{2}$. The upper bound
$\mathcal{A}\le(\gamma/\sigma)^{2}$ under bulk orthogonality follows from
$\kappa_{\mathrm{eff}}\le1$, which in turn follows from $\|P\|_2=1$:
$\|X^{\top}Pe_{\star}\|_2\le\|X\|_2\|Pe_{\star}\|_2\le\|X\|_2$, and in the
orthogonal-bulk case $\|X^{\top}p_{\star}\|_2^{2}
=\kappa^{2}\|x^{\star}\|_2^{2}+O(B^{-1})\le\|x^{\star}\|_2^{2}$.
\end{proof}

\subsubsection{Proof of Proposition~\ref{prop:direction}}

\begin{proof}
Substituting $\rho=\rho_{\star}e_{\star}$ into \eqref{eq:wgrad},
\[
\nabla_wL_{\mathrm{BN}}=\frac{a\gamma\rho_{\star}}{\sigma}X^{\top}p_{\star}
=\frac{a\gamma\rho_{\star}}{\sigma}
\Big[\kappa x^{\star}-\frac1B\sum_{j\neq\star}(1+u_jt)x_j\Big],
\]
using the entries of $p_{\star}$ computed in Step~5 of the proof of
Proposition~\ref{prop:exact-margin}. For the no-BN model,
$\nabla_wL_{\mathrm{noBN}}=a\rho_{\star}x^{\star}$. The cosine between the two is
\[
\cos\theta=\frac{\langle X^{\top}p_{\star},x^{\star}\rangle}
{\|X^{\top}p_{\star}\|_2\,\|x^{\star}\|_2}
=\frac{\kappa\|x^{\star}\|_2^{2}-\tfrac1B\sum_{j\neq\star}(1+u_jt)\langle x_j,x^{\star}\rangle}
{\|X^{\top}p_{\star}\|_2\,\|x^{\star}\|_2},
\]
which reduces to $\kappa\|x^{\star}\|_2/\|X^{\top}p_{\star}\|_2$ when the bulk is
orthogonal to $x^{\star}$, and equals $1$ iff
$X^{\top}p_{\star}\parallel x^{\star}$, i.e.\ iff the back-reaction term vanishes.
\end{proof}

\begin{remark}[Why direction matters for the memorization claim]
\label{rem:direction-matters}
A pure rescaling of the no-BN gradient would leave every normalized update
direction unchanged and could not, by itself, distinguish memorization from
faster feature learning. Proposition~\ref{prop:direction} shows BN does not
merely rescale: the back-reaction term is nonzero even when
$\rho_j\approx0$ for all $j\neq\star$, so BN routes the outlier's error signal
onto the bulk's input directions through $\mu$ and $\sigma$. Empirically we
therefore report, alongside gradient norms, the alignment
\[
\cos\Big(\nabla_w\ell_{\star},\ \tfrac{1}{|\mathcal{B}|-1}\textstyle\sum_{j\neq\star}\nabla_w\ell_j\Big),
\]
predicting \emph{lower} alignment under BN (more outlier-specific updates) rather
than merely larger norms.
\end{remark}

\begin{remark}[Per-sample gradients are not batch-independent under BN]
\label{rem:dpsgd}
Equation~\eqref{eq:wgrad} shows that a sample's contribution to $\nabla_wL$
depends on the whole batch through $P$. Under BN there is therefore no
batch-independent "per-sample gradient" to clip, which complicates the
per-sample clipping step of DP-SGD beyond the usual bookkeeping; see
\S\ref{app:limitations}.
\end{remark}



\subsection{Why $\gamma/\sigma>1$ emerges}
\label{app:gamma-sigma}

\subsubsection{Proof of Proposition~\ref{prop:gamma-exact}}

\begin{proof}
The normalized activations $u_i$ are functions of $w$, $b$ and the batch only;
they do not depend on $\gamma$. Hence, from $z_i=a(\gamma u_i+\beta)+c$,
\[
\frac{\partial L}{\partial\gamma}
=\sum_i\frac{\partial\ell_i}{\partial z_i}\frac{\partial z_i}{\partial\gamma}
=a\sum_i\rho_iu_i=aB\,\overline{\rho u},
\]
with no approximation of any kind: the fixed-statistics simplification concerns
$\partial u/\partial h$, which does not enter here. For the tail sample,
$|\partial\ell_{\star}/\partial\gamma|=|\rho_{\star}||a||t|
=\varphi(-m_{\star})|a||t|=\Theta(|t|)$ for as long as the sample remains unfit
($\varphi(-m_{\star})$ bounded away from $0$). For a typical sample either
$u_i\approx0$ (activation near the mean) or $\varphi(-m_i)\approx0$
(well classified), so $|\partial\ell_i/\partial\gamma|\approx0$; in expectation
over a mini-batch of typical samples, $\overline{\rho u}$ is the empirical
correlation between the loss residual and the normalized activation, which is
$O(B^{-1/2})$ when the two are uncorrelated.
\end{proof}

\subsubsection{Proof of Proposition~\ref{prop:sigma-bulk}}

\begin{proof}
By definition $\sum_i(h_i-\mu)^{2}=B\sigma^{2}$ and $(h_{\star}-\mu)^{2}
=t^{2}\sigma^{2}$, so the tail sample's share is exactly $t^{2}/B$, bounded by
$(B-1)/B$ via Lemma~\ref{lem:leverage}. Writing
$B\sigma^{2}=t^{2}\sigma^{2}+\sum_{j\neq\star}(h_j-\mu)^{2}$ and
$\sigma_{\mathrm{bulk}}^{2}:=\tfrac{1}{B-1}\sum_{j\neq\star}(h_j-\mu)^{2}$ gives
$\sigma^{2}=\frac{B-1}{B-t^{2}}\,\sigma_{\mathrm{bulk}}^{2}$, hence
$\sigma=\sigma_{\mathrm{bulk}}\big(1+O(t^{2}/B)\big)$ for $t^{2}\ll B$. At
$B=256$: $t=3$ gives a $3.5\%$ share and $t=5$ a $9.8\%$ share.
\end{proof}

\begin{remark}[The population asymmetry, stated precisely]
\label{rem:asymmetry}
Propositions~\ref{prop:gamma-exact} and \ref{prop:sigma-bulk} show that the
numerator and denominator of the amplification ratio are driven by disjoint
sample populations. The drive on $\gamma$ from a single unfit tail sample is
$\Theta(|t|)$ and does not average away, whereas its contribution to $\sigma$ is
$O(t^{2}/B)$ and is further diluted as $B$ grows. Consequently $\gamma$ rises
while $\sigma$ is pinned by the bulk, and $\gamma/\sigma$ increases with
\emph{memorization load}. The empirical counterpart of this statement, and the
one we test, is that $\gamma/\sigma$ should increase monotonically with the
corruption ratio $k$ within a fixed architecture, and should remain near its
initialization when no memorizable tail is present. This converts
$\gamma/\sigma>1$ from an assumption into a prediction of the mechanism.
\end{remark}

\subsubsection{Proof of Theorem~\ref{thm:loop}}

Throughout this proof $n$ is the training-step index and $t=u_{\star}$ the
tail deviation; $\alpha,\beta$ are as in \eqref{eq:alphabeta}.

\begin{proof}
\emph{Part (i).} Gradient descent on $\gamma$ using the tail sample gives
$\gamma_{n+1}=\gamma_{n}-\eta\,\partial\ell_{\star}/\partial\gamma$.
By Proposition~\ref{prop:gamma-exact},
$\partial\ell_{\star}/\partial\gamma=-y^{\star}\varphi(-m_{n})\,a\,t$. When
the label, the classifier coefficient and the deviation direction align, i.e.\
$y^{\star}a\,t>0$ (the model has associated the channel's activation direction
with the label), we have $\partial\ell_{\star}/\partial\gamma<0$ and
\[
\Delta\gamma_{n}=-\eta\frac{\partial\ell_{\star}}{\partial\gamma}
=\eta\,\varphi(-m_{n})\,|a|\,|t|
\;=\;\beta\,\varphi(-m_n)\;>0 .
\]

\emph{Part (ii).} By Proposition~\ref{prop:sigma-bulk} the tail sample perturbs
$\sigma$ only at $O(t^{2}/B)$, and the bulk contribution to $\sigma$ stabilizes
once bulk features converge, so $\sigma$ is stationary to that order over the
horizon on which $\gamma$ grows. Then $\gamma_{n}\uparrow$ implies
$(\gamma_{n}/\sigma)^{2}\uparrow$. The attenuation $\kappa_{\mathrm{eff}}$
in Corollary~\ref{cor:amplification} depends only on $t$, $B$ and the batch
geometry $K$, not on $\gamma$, so the amplification
$\mathcal{A}=(\gamma_{n}/\sigma)^{2}\kappa_{\mathrm{eff}}^{2}$ inherits the
monotonicity, and by Proposition~\ref{prop:exact-margin} $\Delta m$ increases at
the next step, i.e.\ $\Delta m_n=\alpha\gamma_n^{2}\varphi(-m_n)$. This
closes the loop \eqref{eq:loop}. Note that
$\kappa_{\mathrm{eff}}$ enters as a constant factor of the loop gain and, being
strictly positive whenever $t^{2}<B-1$ (Lemma~\ref{lem:leverage}), cannot change
its sign.

\emph{Part (iii), quiescence.} By (i) the sequence $\gamma_n$ is increasing,
so $\Delta m_n=\alpha\gamma_n^{2}\varphi(-m_n)\ge\alpha\gamma_0^{2}\varphi(-m_n)$
and the margin dominates the constant-gain dynamics of Lemma~\ref{lem:monotone}
with $c=\alpha\gamma_0^{2}$. Hence $m_n$ is strictly increasing with
$m_n\to+\infty$, so $\varphi(-m_n)\to0$ and both increments
$\Delta\gamma_n=\beta\varphi(-m_n)$ and
$\Delta m_n=\alpha\gamma_n^{2}\varphi(-m_n)$ tend to $0$.

\noindent\emph{Part (iii), the invariant.} Dividing the two update equations, the
common factor $\varphi(-m_n)$ cancels \emph{identically}, so in the
continuous-time limit
\[
\frac{d\gamma}{dm}
=\frac{\beta\,\varphi(-m)}{\alpha\gamma^{2}\varphi(-m)}
=\frac{\beta}{\alpha\gamma^{2}}
\quad\Longrightarrow\quad
\alpha\gamma^{2}\,d\gamma=\beta\,dm,
\]
and integrating gives \eqref{eq:invariant}. The cancellation of $\varphi$ makes
the $\gamma$--$m$ relation independent of the loss profile: \eqref{eq:invariant}
holds for any strictly positive loss derivative, not only the logistic sigmoid.

\emph{Part (iii), divergence.} For $m\gg0$ we have $\varphi(-m)\simeq
e^{-m}$, and substituting $\gamma_n^{2}=\Theta(m_n^{2/3})$ from
\eqref{eq:invariant} into $\Delta m_n=\alpha\gamma_n^{2}\varphi(-m_n)$ gives
$e^{m}m^{-2/3}\,dm=\Theta(1)\,dn$, hence $e^{m_n}m_n^{-2/3}=\Theta(n)$ and
\[
m_n=\log n+\tfrac23\log\log n+O(1)=\log n+O(\log\log n),
\]
which is the post-saturation regime of Proposition~\ref{prop:multistep}.
Substituting into \eqref{eq:invariant} yields \eqref{eq:gamma-rate}. Consequently
\[
\varphi(-m_n)=\Theta\big(n^{-1}(\log n)^{-2/3}\big),
\qquad
\sum_{n}\Delta\gamma_n=\beta\sum_{n}\varphi(-m_n)=\infty ,
\]
since $\sum_n n^{-1}(\log n)^{-2/3}$ diverges ($2/3<1$). The claim in the earlier
version of this theorem that the loop ``self-terminates at a steady state'' is therefore false. The error was to infer summability of $\Delta\gamma_n$ from the exponential decay of $\varphi$: $\varphi(-m_n)$ does decay exponentially in the \emph{margin}, but the margin grows only logarithmically in the \emph{step count}, which leaves a harmonic-type tail.

\emph{Part (iii), threshold crossing.} $\gamma_n$ is strictly increasing and
unbounded, so for any $\sigma>0$ the set $\{n:\gamma_n>\sigma\}$ is a nonempty
up-set; let $N$ be its least element. Inverting \eqref{eq:gamma-rate},
\[
N\;\approx\;\exp\!\Big(m_0+\tfrac{\alpha}{3\beta}\big(\sigma^{3}-\gamma_0^{3}\big)\Big),
\]
which is finite and computable from $(\eta,a,t,\sigma,\kappa_{\mathrm{eff}},
\|x^{\star}\|_2)$. If $\sigma\le\gamma_0$ then $\gamma_0/\sigma\ge1$ already and
$N=0$ by monotonicity; with the standard initialization $\gamma_0=1$ this covers
every specialized channel whose activation spread is at most the initialization
scale.
\end{proof}


\begin{remark}[Gauge caveat, restated]
As noted in Remark~\ref{rem:gauge}, BN's invariance to $\|w\|$ means the absolute
magnitude of $\gamma/\sigma$ is not scale-free: rescaling $w\mapsto cw$ leaves
the function unchanged while scaling $\sigma\mapsto c\sigma$. Part (iii)'s
conclusion $\gamma_n/\sigma>1$ for $n\ge N$ is therefore stated relative to a fixed initialization gauge. The gauge-invariant content of the theory is the matched BN vs.\ no-BN ratio of Corollary~\ref{cor:amplification}, in which the parametrization cancels, together with the monotone dependence of $\gamma/\sigma$ on corruption ratio within a fixed run.
\end{remark}

\subsection{Multi-step memorization speedup}
\label{app:multistep}

\subsubsection{Discrete dynamics}

From Proposition~\ref{prop:exact-margin}, the margin trajectory on a fixed tail
sample at fixed $\gamma$ obeys
\begin{equation}
\begin{aligned}
m_{n+1}
&= m_{n}+c\,\varphi(-m_{n}),
\qquad
c=\eta\,s^2\kappa_{\mathrm{eff}}^2\|x^{\star}\|_2^2,\\
c_{\mathrm{BN}}
&=\Big(\frac{\gamma}{\sigma}\Big)^2
  \kappa_{\mathrm{eff}}^2\,c_{\mathrm{noBN}} .
\end{aligned}
\label{eq:discrete}
\end{equation}
\noindent Here $n$ is the step index; the tail deviation $t$ enters only through $\kappa_{\mathrm{eff}}$.

\begin{lemma}[Monotone divergence]
\label{lem:monotone}
For any $c>0$ and any $m_0\in\mathbb{R}$, the sequence \eqref{eq:discrete} is
strictly increasing and $m_{n}\to+\infty$. The same holds for any dynamics $m_{n+1}=m_n+c_n\varphi(-m_n)$ with $c_n\ge c>0$, by comparison.
\end{lemma}

\begin{proof}
$\varphi(-m_{n})>0$ for all finite $m_{n}$, so
$\Delta m_{n}=c\varphi(-m_{n})>0$ and the
sequence is strictly increasing. Suppose $m_{n}\to M<\infty$. Then
$\varphi(-m_{n})\to\varphi(-M)>0$, so
$\Delta m_{n}\to c\varphi(-M)>0$, contradicting
convergence to a finite limit. Hence $m_{n}\to+\infty$.
For the comparison statement, $\Delta m_n\ge c\varphi(-m_n)>0$ and the same contradiction applies.
\end{proof}

\subsubsection{Proof of Proposition~\ref{prop:multistep}}

\begin{proof}
When $c$ is small relative to the margin scale, \eqref{eq:discrete} is
well approximated by the ODE $\dot m=c\,\varphi(-m)=c/(1+e^{m})$. Separating
variables and integrating from $(m_0,0)$ to $(M,T)$,
\begin{equation}
\begin{aligned}
\int_{m_0}^{M}(1+e^{m})\,dm
&= \int_0^{T}c\,dn = cT
\quad\Longrightarrow\\
T(c;m_0,M)
&= \frac{1}{c}
\Big[(M+e^{M})-(m_0+e^{m_0})\Big].
\end{aligned}
\end{equation}
The bracket depends only on $m_0$ and $M$, not on $c$ or $s$, so
\[
\frac{T_{\mathrm{noBN}}}{T_{\mathrm{BN}}}
=\frac{c_{\mathrm{BN}}}{c_{\mathrm{noBN}}}
=\frac{(a\gamma/\sigma)^{2}\kappa_{\mathrm{eff}}^{2}}{a^{2}}
=\Big(\frac{\gamma}{\sigma}\Big)^{2}\kappa_{\mathrm{eff}}^{2},
\]
independently of $m_0$, $M$, $\|x^{\star}\|_2$ and $\eta$.
\end{proof}

\begin{remark}[Robustness of the scaling]
The speedup holds for any target margin $M>m_0$, whether the model is in the
pre-saturation regime ($m\ll0$), near the decision boundary ($m\approx0$), or
confidently memorized ($m\gg0$). Asymptotically: for $m\ll0$, $\varphi(-m)\approx1$
and the margin grows linearly, $m_{n}\approx m_0+cn$, a
$(\gamma/\sigma)^{2}\kappa_{\mathrm{eff}}^{2}$ rate advantage; for $m\gg0$,
$\varphi(-m)\approx e^{-m}$ and $m(n)\approx\log(cn+e^{m_0})$, so the
BN model retains a persistent offset
$m_{\mathrm{BN}}(n)-m_{\mathrm{noBN}}(n)\approx2\log\!\big((\gamma/\sigma)\kappa_{\mathrm{eff}}\big)$.
This $m_n=\log n+O(1)$ behaviour at fixed $\gamma$ is what
Theorem~\ref{thm:loop}(iii) refines to $m_n=\log n+O(\log\log n)$ once the slow
growth of $\gamma_n$ is taken into account.
This two-regime behaviour matches Figure~3: early on, BN models reduce loss on
noisy samples dramatically faster; later, both converge but BN reaches lower
loss.
\end{remark}

\subsection{Discussion, validation and limitations}
\label{app:limitations}

\paragraph{Why outliers are disproportionately affected.}
Three compounding mechanisms, all now expressed in exact quantities:
\begin{enumerate}\itemsep2pt
\item \emph{Persistent tail activations.} Outliers maintain large $|t|$
throughout training. In \eqref{eq:gamma-grad} the drive on $\gamma$ is
$\Theta(|t|)$, whereas typical samples with $u_i\approx0$ contribute $\approx0$.
\item \emph{Channel specialization.} Channels with small $\sigma$ correspond to
specialized, low-entropy features; outliers activate them atypically, and these
channels produce the largest $\gamma/\sigma$. Our layerwise analysis (Figure~10)
confirms early layers exhibit the largest ratios.
\item \emph{Self-reinforcing dynamics.} Theorem~\ref{thm:loop}: outliers drive
$\gamma$ upward, which amplifies their own memorization. The loop is absent in
no-BN models, where no analogue of $\gamma$ exists. The loop does not
terminate at a finite $\gamma_{\infty}$, but its residual drive decays like
$n^{-1}(\log n)^{-2/3}$ and is negligible well before the end of training
(Corollary~\ref{cor:finite-horizon}).
\end{enumerate}


\paragraph{Multi-layer amplification.}
In an $L$-layer network each BN layer applies its own projector $P^{(l)}$ and
ratio $\gamma_l/\sigma_l$. The backward gradient at the input accumulates
contributions from all of them, with effective amplification involving terms of
the form $\prod_{l}(\gamma_l/\sigma_l)\kappa^{(l)}_{\mathrm{eff}}$. Since
$\|P^{(l)}\|_2=1$, each layer's projector is a contraction, so the product form
gives an upper bound; our observation that $\gamma/\sigma$ is largest in early
layers (Figure~10) suggests the total amplification exceeds what a single-layer
analysis predicts.

\paragraph{Implications for differential privacy.}
Two consequences follow. First, BN amplifies outlier gradient norms by
$(\gamma/\sigma)^{2}\kappa_{\mathrm{eff}}^{2}$, so in DP-SGD either the clipping
threshold $C$ must be raised (increasing the noise required for the same
guarantee), or disproportionately more outlier information is destroyed by
clipping. Second, and more fundamentally,
Remark~\ref{rem:dpsgd} shows that under BN a sample's gradient contribution is
not independent of its batch, so the per-sample clipping primitive that DP-SGD
assumes is not cleanly available. This suggests a structural, not merely
quantitative, tension between BN and DP training.

\begin{algorithm} [t]
	\SetKwInOut{Input}{Input}
	\SetKwInOut{Output}{Output}
	\Input{Dataset $D$, classes set $C$, flipping ratio $k$}
	\Output{Set $D'$ that contains label flipping instances}
    Initialize $D'=\emptyset$; \\
    $\delta=|D|\cdot k$; \\
    \While{$|D'|<\delta$}{
    randomly pick $(x,y)\in D$; \\
    randomly pick $y^*\in C/\{y\}$; \\
    $D'= D'.\text{insert}(\{(x,y^*)\})$; \\
    $D=D/\{(x,y)\}$; \\
    }
    $D'=D'.\text{union}(D)$; \\
    \Return $D'$ \;
	\caption{Corrupted label data synthesis}
    \label{ap:alg:hard_memorization}
\end{algorithm}

\begin{algorithm} [t]
	\SetKwInOut{Input}{Input}
	\SetKwInOut{Output}{Output}
	\Input{Dataset $D$, other dataset $D^*$, label class $C$, noisy data ratio $k$}
	\Output{Set $D'$ that contains noisy instances}
    randomly pick a specific label $y\in C$; \\
    randomly select $X^*=\{x : (x, .)\in D^*\}$; \\
    $\delta=|D|\cdot k$; \\
    initialize $D'=\emptyset$; \\
    initialize $i=0$; \\
    \While{$i < \delta$}{
        randomly pick $x\in X^*$; \\
        $D'=D'.\text{insert}(\{(x,y)\})$; \\
        $i \leftarrow i + 1$;
    }
    $D'=D'.\text{union}(D)$;\\
    \Return $D'$ \;
	\caption{OOD data synthesis}
    \label{ap:alg:soft_memorization}
\end{algorithm}
\section{Corrupted data synthesis pseudos}
\label{sec:ap:pseudos}
We perform corrupted data and out-of-distribution data generation. Details of these algorithms are shown in Algorithm \ref{ap:alg:hard_memorization} and Algorithm \ref{ap:alg:soft_memorization}.


\section{Detailed experiments analysis}
\label{sec:ap:detailed_analysis}

\subsection{Memorization under controlled budgets}
\noindent\textbf{Disentangling the BN from optimization budget, testing loss, and utility confound.} We compare models incorporating BN and without BN that have the same computational budget, including optimization steps, testing loss, and utility. Figure \ref{fig:fixed_control} reveals that BN models consistently memorize corrupted samples with higher values than no-BN models over diverse datasets. Specifically, we compute the margin of memorization between BN models and no-BN models given the same budget, and then we report the margin distribution in Figure \ref{fig:fixed_control_hist}. These distributions skew toward positive values, indicating that BN models have a higher capacity for noisy data memorization than no-BN models.

\begin{figure}[!h]
    \centering
    \includegraphics[width=\columnwidth]{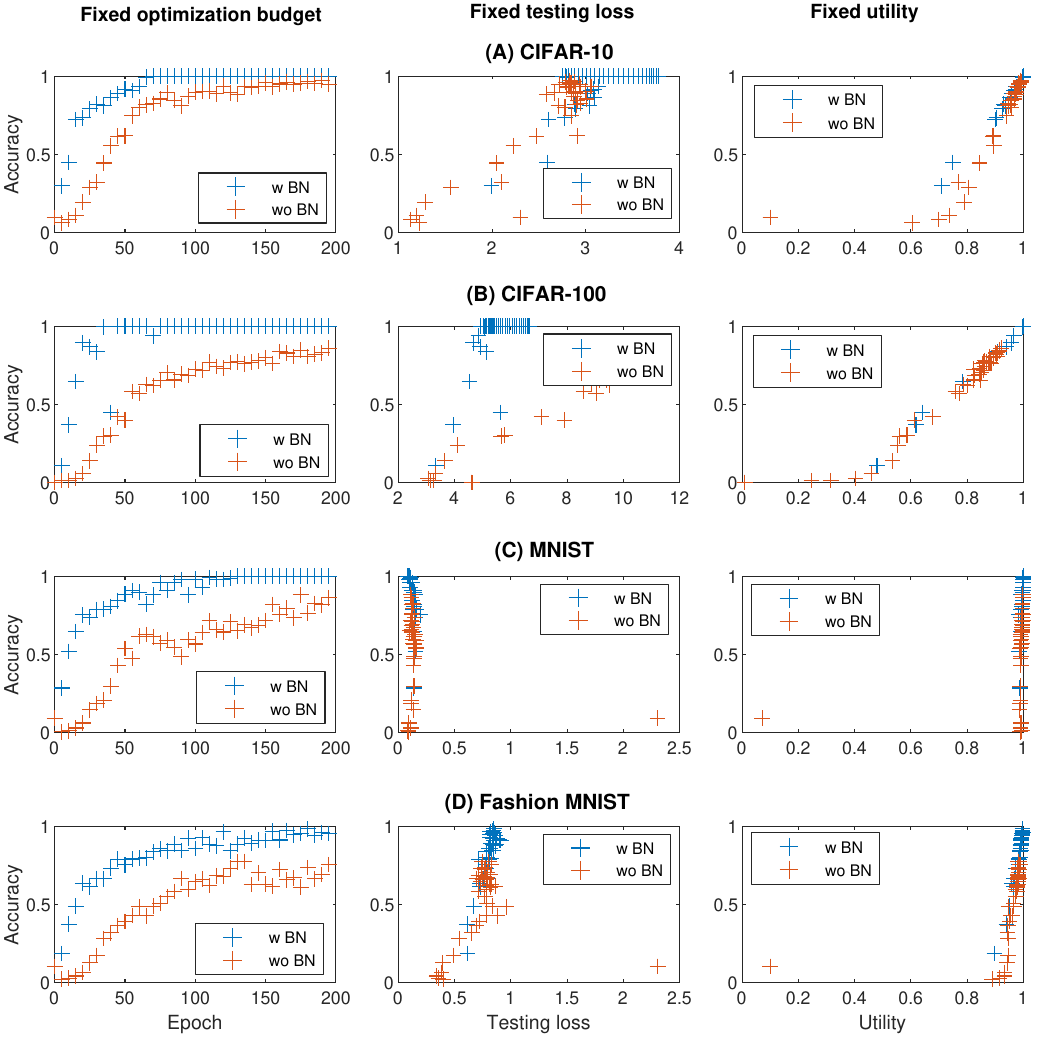}
    \caption{Distribution of the difference between BN models and no-BN models given a fixed budget of aspects.}
    \label{fig:fixed_control}
\end{figure}

\begin{figure}[!h]
    \centering
    \includegraphics[width=\columnwidth]{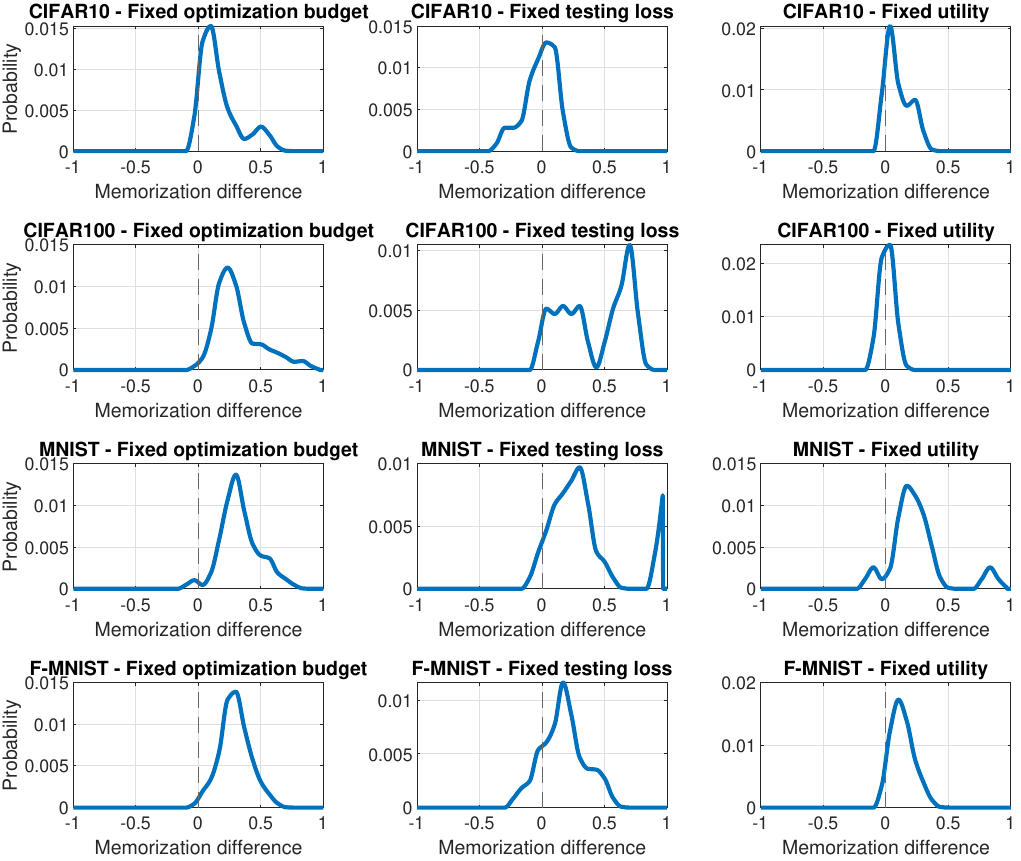}
    \caption{Distribution of the difference between BN models and no-BN models given a fixed budget of aspects.}
    \label{fig:fixed_control_hist}
\end{figure}

\noindent\textbf{Disentangling the BN from a loss confound.} To disentangle the impact, we analyze the loss-normalized influence in Figure \ref{fig:loss-matched}, which shows consistent results: a higher BN loss-normalized gradient for tail samples. We record the per-sample loss at the same training checkpoint, time-averaged over training. Pooling samples across BN and NoBN models, we compute the mean gradient norm separately for each model.
\begin{figure}[!h]
    \centering
    \vspace{-10pt}
    \includegraphics[width=0.5 \textwidth]{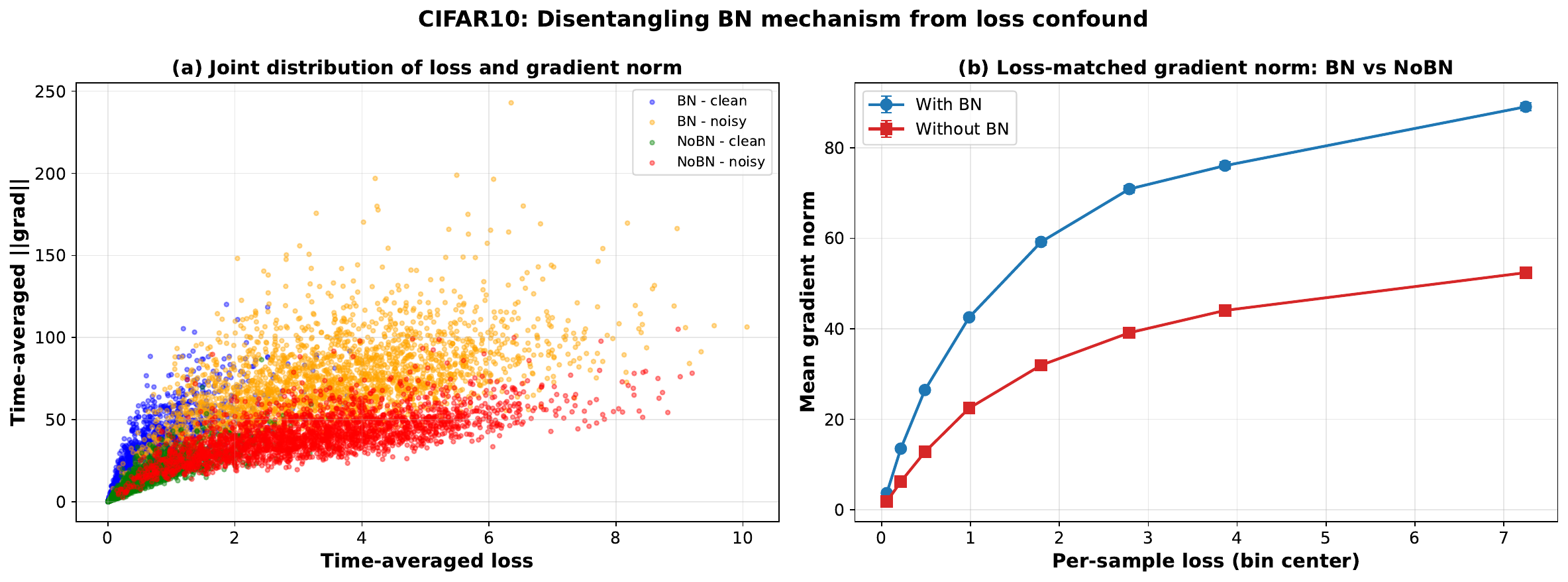}
    \vspace{-10pt}
    \caption{\textbf{Disentangling the BN from a loss confound.}
    (a)~Joint distribution of time-averaged per-sample loss and grad norm
    on CIFAR-10 (k=5\%); 
    (b)~Loss-matched comparison of mean grad norm.}
    \vspace{-10pt}
    \label{fig:loss-matched}
\end{figure}

\subsection{\texorpdfstring{Extended analysis of $\gamma/\sigma$ ratios}
{Extended analysis of gamma/sigma ratios}}
\label{sec:ap:bn_params}
Figure \ref{fig:ap:parameter_ratio} shows the ratio between the scale $\gamma$ and the variance parameter $\sigma$ of BN in Lenet, training across multiple datasets and multiple noisy factor $k$. Over multiple settings, the results consistently confirm that the ratio $\gamma/\sigma$ is much greater than 1 on all model layers. Connected with Proposition~3, the ratio highlights the strong gradient norm amplification on corrupted samples, therefore increasing the model memorization of these data points.

\subsection{Learning characteristics analysis}
\label{subsec:ap:learning_curve}

Figure \ref{fig:ap:loss} presents training loss curves on CIFAR10 of Lenet, Resnet34, Resnet50 over various values of $k\in \{5\%, 10\%, 15\%, 20\%, 25\%, 30\%\}$. There are four curves for each sub-figure, including the accuracy of models with and without BN on a clean set (nonmem-set) and a noisy set (mem-set). Since the trends of models with BN decrease dramatically and below the curves of models without BN, it highlights the fast convergence speed and memorization for models incorporating BN on both the clean and noisy set.

\begin{figure}[t]
    \centering
    \includegraphics[width=\columnwidth]{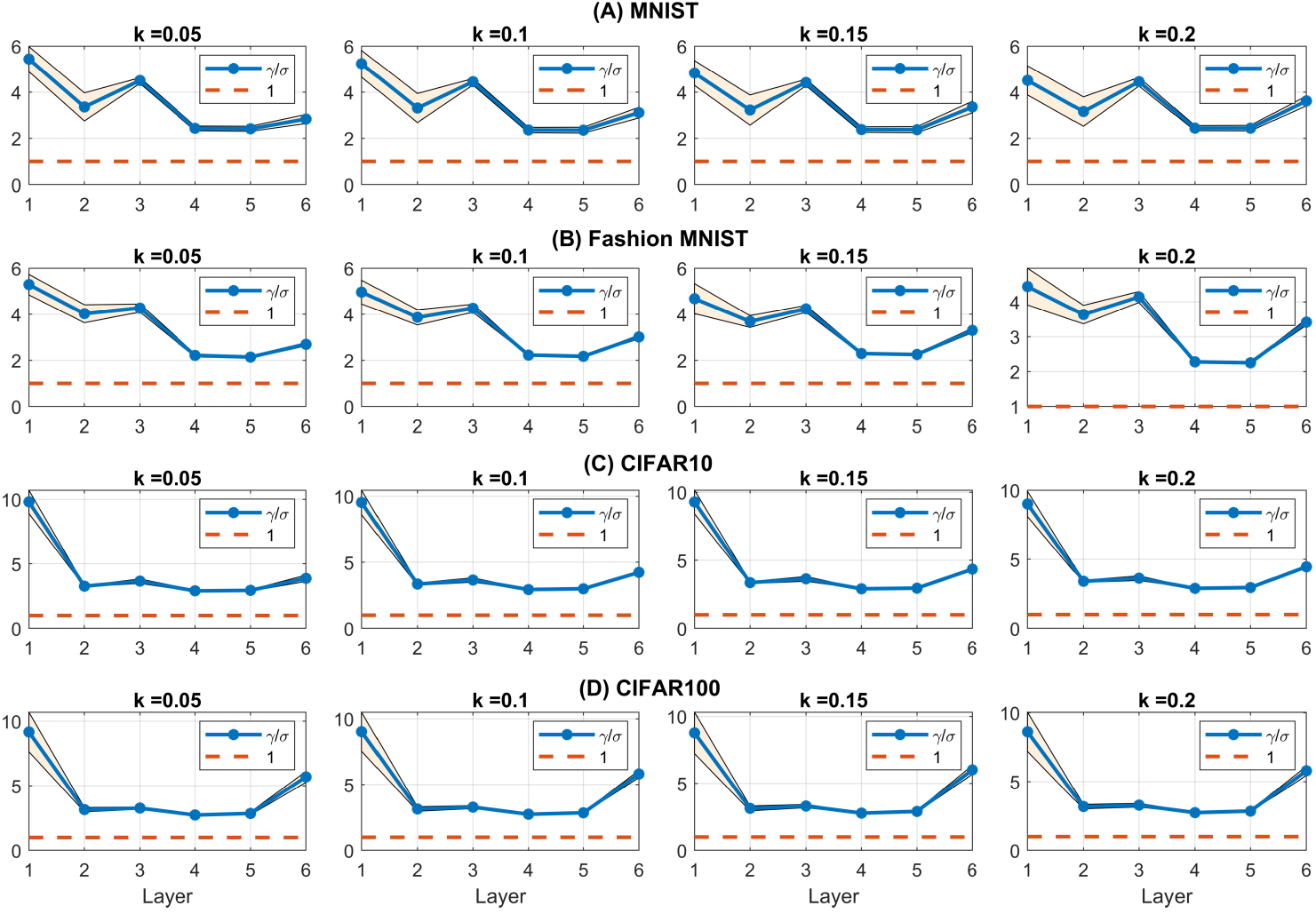}
    \caption{Ratio between two parameters ($\gamma$, $\sigma$) of Batch Norm layers of Lenet training on datasets with $k=0.1$. (Median values of repeated five running times)}
    \label{fig:ap:parameter_ratio}
\end{figure}

\begin{figure}[t]
    \centering
    \includegraphics[width=1\columnwidth]{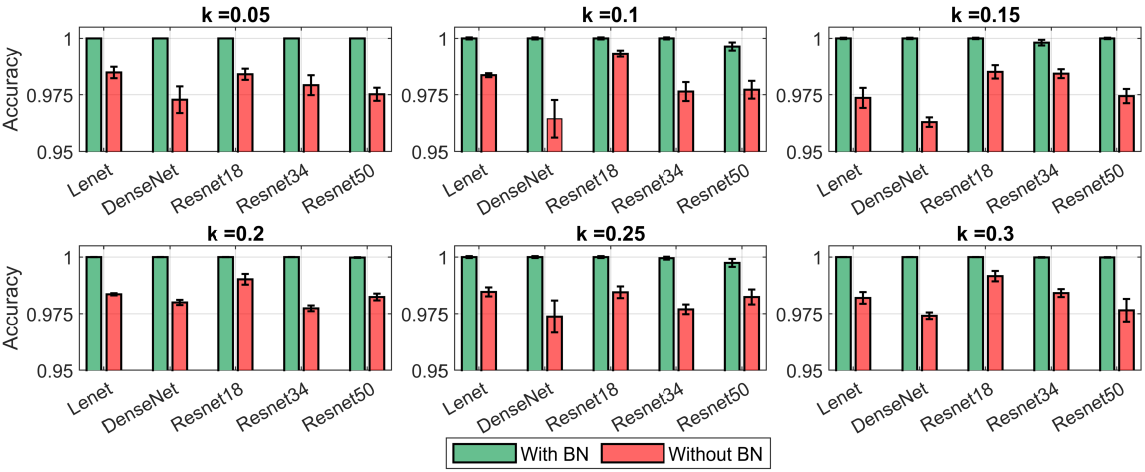}
    \caption{Comparison of memorization performance of model training with and without BN on natural OOD data across different neural architectures.}
    \label{fig:ap:ood_acc}
\end{figure}

\subsection{Extended analysis of memorization on OOD data}
\label{subsec:ap:diff_hyperparams}
Figure \ref{fig:ap:ood_acc} shows the memorization capacity of the BN model and the no-BN model in OOD data distributions with respect to various values $k\in \{0.05, 0.1, 0.15, 0.2, 0.25, 0.3\}$ with different model architectures. Generally, our observation is consistent across OOD ratios, where the performance of models incorporating BN is higher than that of models without BN by more than $3\%$. It indicates that BN increases the memorization effects on the OOD set.

\subsection{MIA on baseline datasets (without OOD)}
\label{subsec:mia_w/o}
\noindent\textbf{Algorithm description.} Along with the study of the privacy attack on OOD datasets, we further explore the sensitivity of models integrating BN solely on clean datasets. Intuitively, this setting encourages the MIA algorithm to concentrate on samples that are naturally memorized by the model's viewpoint rather than manually created noisy samples by data viewpoints. Specifically, ResNet-based architectures are trained on the clean CIFAR-10 dataset. After that, we exploit the LIRA attack and measure the success rate by AUC metrics to compare the vulnerable leakage between BN models and no-BN models.\\

\noindent\textbf{Results.} Our results in Figure \ref{fig:no_mem_mia} illustrate that even in the clean setting, models including BN achieve higher attacking AUC at least $1\%$ than models excluding BN. This result indicates that models incorporating BN are more privacy vulnerable against MIA than models with no BN.

\begin{figure}[t]
    \centering
    \includegraphics[width=0.5\linewidth]{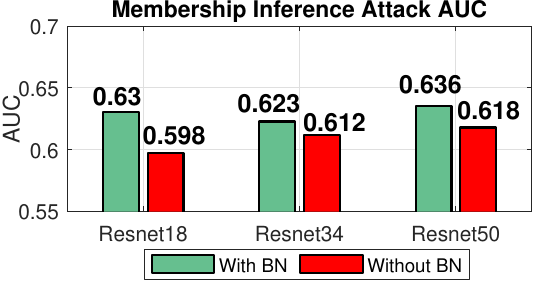}
    \caption{Comparison of memorization performance of model training with and without BN on the clean dataset across different neural architectures.}
    \label{fig:no_mem_mia}
\end{figure}

\begin{figure}[t]
    \centering
    \includegraphics[width=\columnwidth]{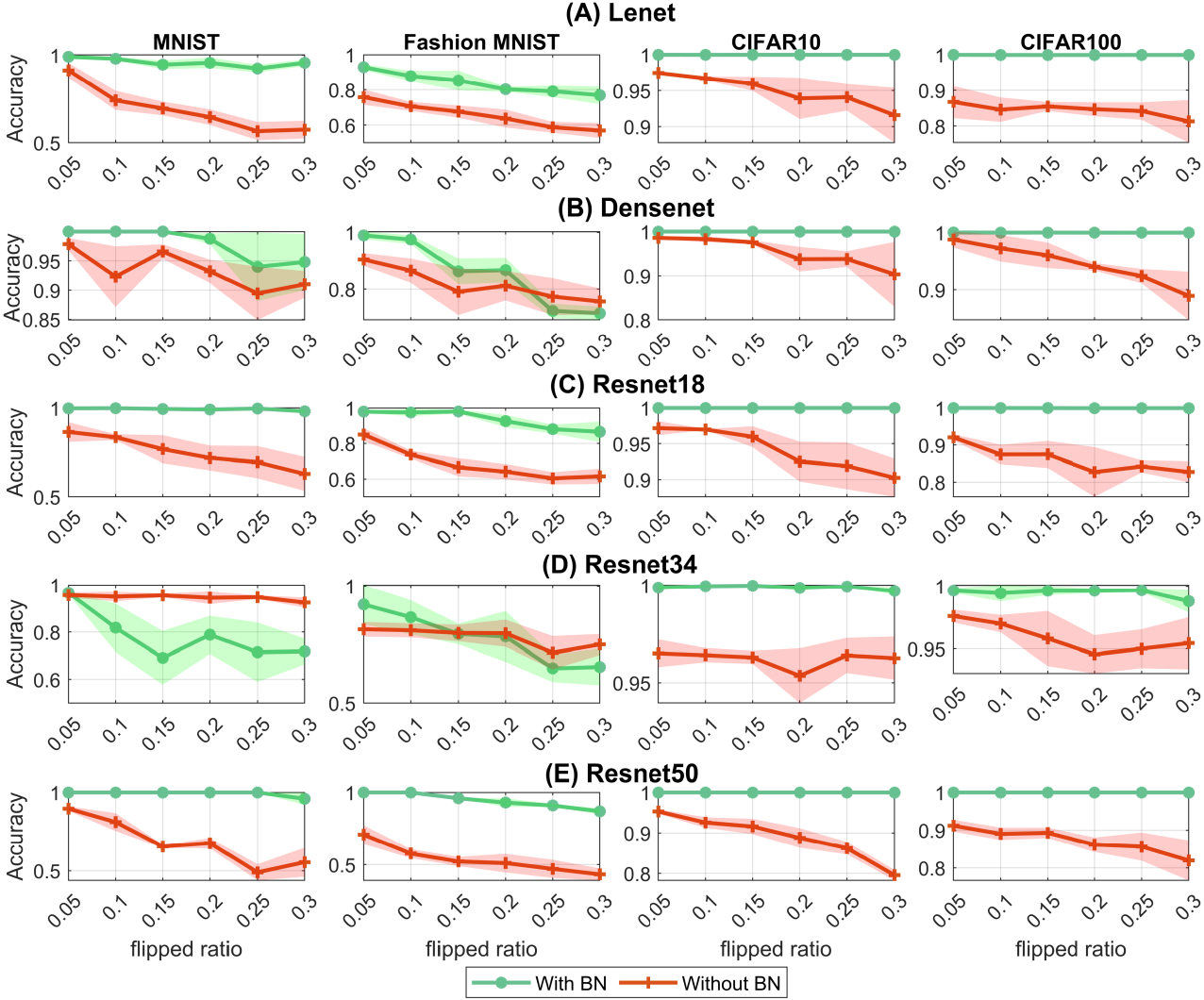}
    \caption{Accuracy of Lenet, DenseNet, Resnet18, Resnet34 with and without BN on noisy set of MNIST, FashionMNIST, CIFAR10 and CIFAR100 using $10^{-3}$ learning rate (Best viewed in color)}
    \label{fig:ap:hard_memorization}
\end{figure}

\subsection{MIA on models given the same clean accuracy}
\textcolor{blue}{Figure~\ref{fig:mia_clean} compares MIA performance on BN models and no-BN models with the same clean accuracy. No-BN models with clean accuracy close to BN models show lower privacy leakage than BN models. Therefore, the privacy gap remains even with the same clean accuracy control.}

\subsection{MIA results statistical comparison}
Figure ~\ref{fig:mia_stats} shows that the variance in MIA across architectures and ratios is similar between BN and no-BN. \textcolor{blue}{To make a robust comparison, we conduct a one-tailed t-test with the hypothesis as follows: \\
- $H_0$: MIA performs equally on BN and no-BN models. \\
- $H_1$: MIA performs worse on no-BN than BN models. \\
Specifically, there are five neural architectures and six corrupted ratios, leading to total 30 results for each group with 29 degrees of freedom. It yields a p-value of 0.11\% ($<$ 5\%), rejecting $H_0$ and accepting $H_1$.}

\section{Experimental setup details}
\label{sec:ap:implementation_details}
\vspace{2mm}\noindent\textbf{Hyperparameters.} We train all models for 200 epochs on the respective classification datasets. Adam optimizer is used with a fixed learning rate equal to $1e^{-4}$ and a batch size of $256$. We disable data augmentation and dropout layers to mitigate the randomness during training. To ensure the robustness of our results, each experiment is repeated five times, and we report the mean and standard deviation of the obtained performance metrics.

\vspace{2mm}\noindent\textbf{MIA settings.} We use 32 shadow models with architectures matching the target model to approximate its behavior. As a SOTA MIA method, LIRA motivates our choice over loss-threshold or metric-based attacks to measure privacy leakage in models with and without BN layers.

\vspace{2mm}\noindent\textbf{Evaluation metrics.} We employ different evaluation metrics to analyze the effect under three distinct scenarios. For \textit{noisy data influence}, we use classification accuracy to quantify the degree of memorization of models on noisy datasets. For \textit{membership inference attack influence} setting, the Receiver Operating Characteristic (ROC) curve and Area Under the Curve (AUC) are used to evaluate the memorization effects under an adversarial setting. Finally, we adopt sample-wise gradient metrics for \textit{per-sample influence} setting, which quantify the influence of each sample on the learned model.

\section{Ablation studies}
\label{sec:ap:diff_hyperparams}
\subsection{Effect of learning rates} 
In order to assess the behavior of models incorporating BN on a memorization setting, we conduct an experiment of BN effect on memorization with $lr=1e^{-3}$, which is shown in Figure \ref{fig:ap:hard_memorization}. Divergent models are exploited, such as Lenet (A), Densenet (B), Resnet18 (C), Resnet34 (D), and Resnet50 (E), with varied corrupted label ratios $k\in\{5\%, 10\%, 15\%, 20\%, 25\%, 30\%\}$. The observation is still similar to our finding in the main paper. The no-BN models show low performance on noisy data ($18/20$ cases), hence indicating a weak memorization capacity compared to models with BN layers.

\begin{figure}[t]
    \centering
    \includegraphics[width=\columnwidth]{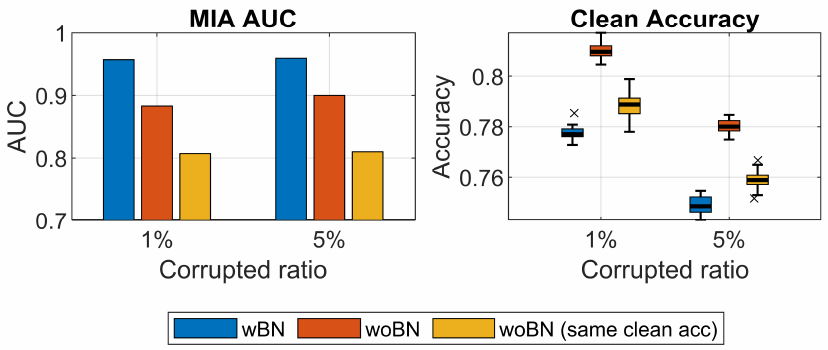}
    \caption{MIA attack on BN models and no-BN models under the same clean accuracy}
    \label{fig:mia_clean}
\end{figure}

\begin{figure}[t]
    \centering
    \includegraphics[width=0.7\columnwidth]{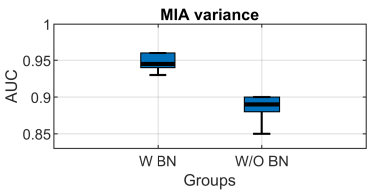}
    \caption{Statistical comparison on MIA between BN models and no-BN models}
    \label{fig:mia_stats}
\end{figure}

\subsection{Effect of batch sizes}

\begin{figure}[!h]
    \centering
    \includegraphics[width=\columnwidth]{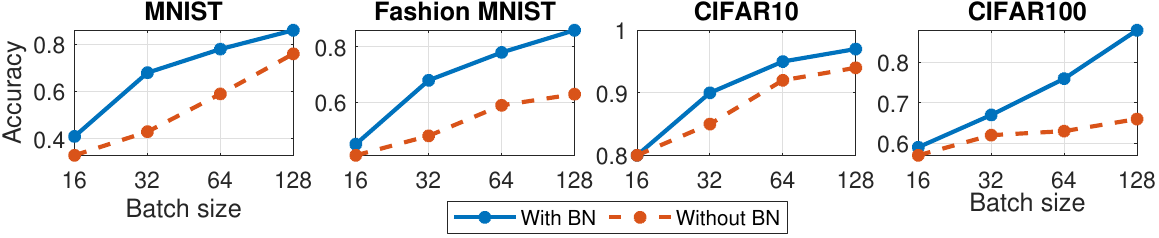}
    \caption{Batch size comparison on ResNet18 (5\%).}.
    \label{fig:batchsize}
\end{figure}
We study the effect of batch size in the forced memorization setting (Fig~\ref{fig:batchsize}). The results confirm increased memorization in BN models across batch sizes.

\subsection{Effect of optimizers}
\begin{figure}[!h]
    \centering
    \includegraphics[width=\columnwidth]{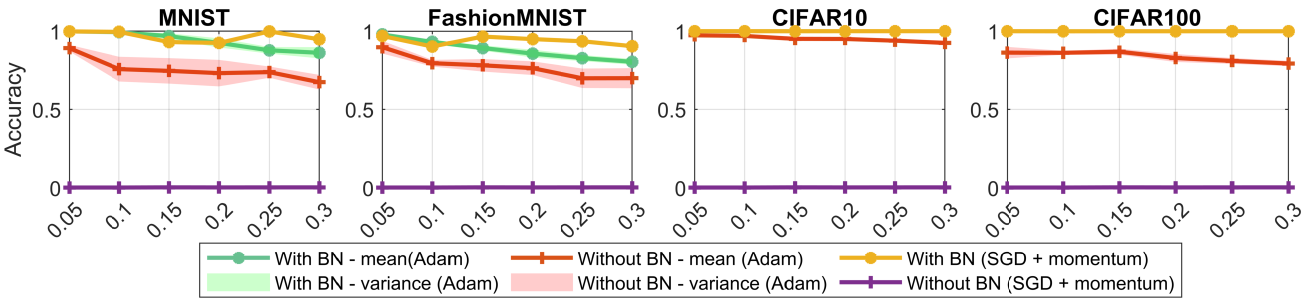}
    \caption{Performance of ResNet18 with different optimizers}
    \label{fig:sgd}
\end{figure}

Figure \ref{fig:sgd} confirms that models with BN (green, yellow lines) have a stronger memorization accuracy than models without BN (red, purple lines).   

\subsection{Effect of normalization techniques}
\begin{figure}[!h]
    \centering
    \includegraphics[width=0.9\linewidth]{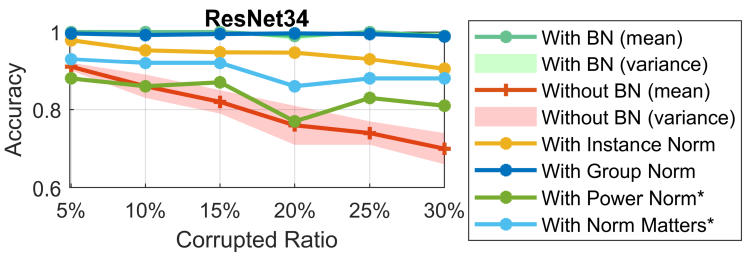}
    \caption{Performance of models with different normalization techniques over multiple corrupted ratios}
    \label{fig:norm}
\end{figure}

\label{subsec:ap:normalization}
We analyze the memorization capacity of ResNet18/34 with GroupNorm (GN), InstanceNorm (IN), LayerNorm (LN), and BatchNorm (BN) on CIFAR-10. Figure~\ref{fig:norm} shows \textbf{all normalization strategies consistently amplify memorization} compared to no normalization across all corruption ratios $k$.

\begin{table*}[t]
\centering
\begin{tabular}{c|c|c|c|c}
\toprule
Method           & MNIST           & FashionMNIST   & CIFAR10   & CIFAR100            \\ \midrule
ViT ($k=1\%$) & \textbf{0.993$\pm$0.001} & \textbf{0.993$\pm$0.003} & \textbf{0.997$\pm$0.002} & \textbf{1.0$\pm$0.0}\\
ViT w/o LN ($k=1\%$) & 0.055$\pm$0.006 & 0.055$\pm$0.01 & 0.259$\pm$0.037 & 0.494$\pm$0.11 \\ \midrule
ViT ($k=5\%$) & \textbf{0.992$\pm$0.001} & \textbf{0.995$\pm$0.003} & \textbf{0.999$\pm$0.002} & \textbf{0.999$\pm$0.0} \\
ViT w/o LN ($k=5\%$) & 0.135$\pm$0.065 & 0.157$\pm$0.098 & 0.212$\pm$0.033 & 0.466$\pm$0.243 \\ \midrule
\end{tabular}
\caption{Performance of ViT with the inclusion and removal of Layer Norm on the corrupted labels dataset.}
\label{tab:ap:vit}
\end{table*}

\subsection{Effect of model architectures}
\label{subsec:ap:model_architectures}
\noindent\textbf{Algorithm description.} In this section, we additionally study the privacy vulnerability of the Layer Normalization (LN) layer, which has been widely used on transformer-based architectures---the backbone of recent foundation models. The key advantage of LN compared to BN is that it is independent on the mini-batch size and consistent between training and testing times, hence stabilizing the training procedures, especially on sequential models. In particular, LN computes mean and standard deviation statistics through individual neurons in the layer rather than through samples of a mini-batch like BN. To analyze the impact of LN on privacy, we adopt Vision Transformer (ViT) with the inclusion and removal of LN and train on the corrupted labels dataset.

\noindent\textbf{Results.} Table \ref{tab:ap:vit} presents the accuracy of ViT with and without LN on corrupted label datasets with the noisy ratio $k\in\{0.01, 0.05\}$. The result shows that ViT with LN is more accurate compared to the one without LN. Models incorporating LN can fit very well on the out-of-distribution data, hence demonstrating the capacity of memorization dependent on LN layers.

\subsection{Effect of complex datasets}
\label{subsec:ap:complex_dataset}
\begin{figure}[!h]
    \centering
    \includegraphics[width=0.8\columnwidth]{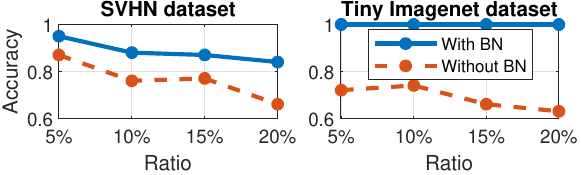}
    \caption{Memorization accuracy in SVHN and Tiny ImageNet}
    \label{fig:complex}
\end{figure}

We report the performance of ResNet34 on two complex datasets: SVHN and TinyImageNet (c.f. Fig~\ref{fig:complex}). These results are consistent with our observation that BN increases memorization.



\section{Trade-off between memorization and generalization}
\label{sec:ap:tradeoff}

\begin{figure}[!h]
    \centering
    \includegraphics[width=\columnwidth]{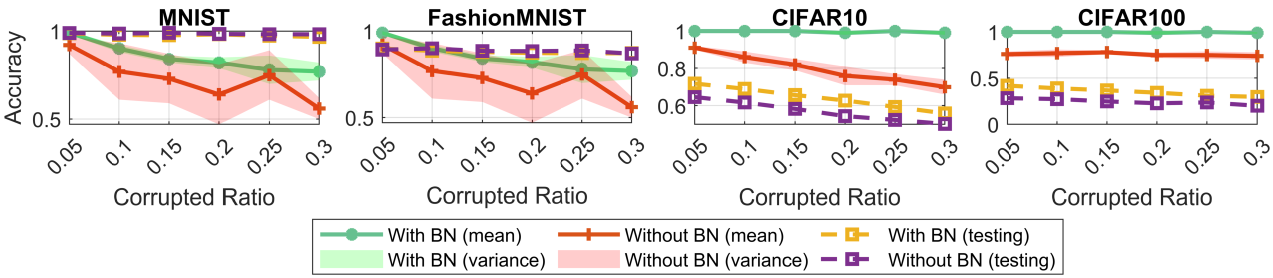}
    \caption{Performance of ResNet34 on various datasets (Best view in color)}
    \label{fig:generalisation}
\end{figure}

\begin{figure}[h!]
    \centering
    \includegraphics[width=0.45\columnwidth]{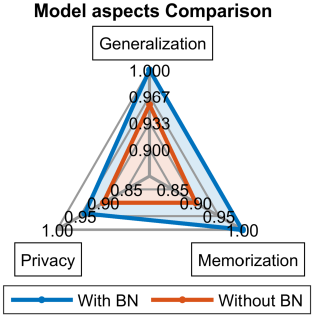}
    \caption{ResNet34 generalization, memorization and MIA performance}
    \label{fig:interplay}
\end{figure}
Figure~\ref{fig:generalisation} shows test accuracy alongside memorization accuracy. The gap in test accuracy is smaller than the memorization gap, indicating that removing BN sacrifices modest generalization while substantially reducing privacy risk. Moreover, Fig.~\ref{fig:interplay} unifies 3 aspects, comparing BN and no-BN. Increased memorization accuracy corresponds to higher privacy vulnerability under MIA but higher accuracy on the clean dataset.

\section{Per-sample influence}
\label{sec:ap:per_sample}

\noindent\textbf{Methodology.} We adopt a per-sample influence analysis framework inspired by prior work on understanding dataset memorization through curvature and gradient-based metrics\cite{koh2017understanding,feldman20}. In particular, we approximate the influence of each training point on the model parameters by computing the $\ell_2$ norm of the gradient of the loss with respect to the model parameters, $|\nabla_\theta \ell(x_i, y_i)|_2$. This measure has been shown to correlate with sample difficulty, out-of-distribution status, and memorization propensity, providing a practical proxy for quantifying the impact of individual data points on training dynamics. We design two experiments on MNIST, FashionMNIST, CIFAR10, and CIFAR100 where $1\%$ and $5\%$ of the training set is corrupted by randomly flipping labels. These mislabeled samples act as atypical instances that the network can only fit by memorization (unintended memorization). For each dataset, two convolutional neural networks with identical architectures are trained to convergence, one including BN layers and the other without BN. After training, we compute per-sample gradient norms for all training points and compare the distributions of clean versus noisy samples across both model variants.

\noindent\textbf{Results.} Our results shown in Figure \ref{fig:grad_dist}  reveal a consistent trend across all datasets. In models with BN, the distribution of gradient norms for noisy samples is significantly shifted towards higher values relative to both clean samples (within the same BN model) and noisy samples in the non-BN model. This indicates that BN accentuates the influence of mislabeled examples, making them stand out more sharply from clean data. In contrast, models without BN exhibit smaller and less distinguishable gradient norms for noisy samples, suggesting weaker memorization signals. These findings are coherent with our previous findings and imply that BN, while beneficial for optimization stability and generalization on clean data, may simultaneously increase the visibility and memorization of atypical samples. Such behavior has important implications for privacy, as amplified memorization signals may increase susceptibility to membership inference and related attacks.

\begin{figure*}[t]
\centering
\begin{subfigure}{0.24\textwidth}
    \centering
    \includegraphics[width=\linewidth]{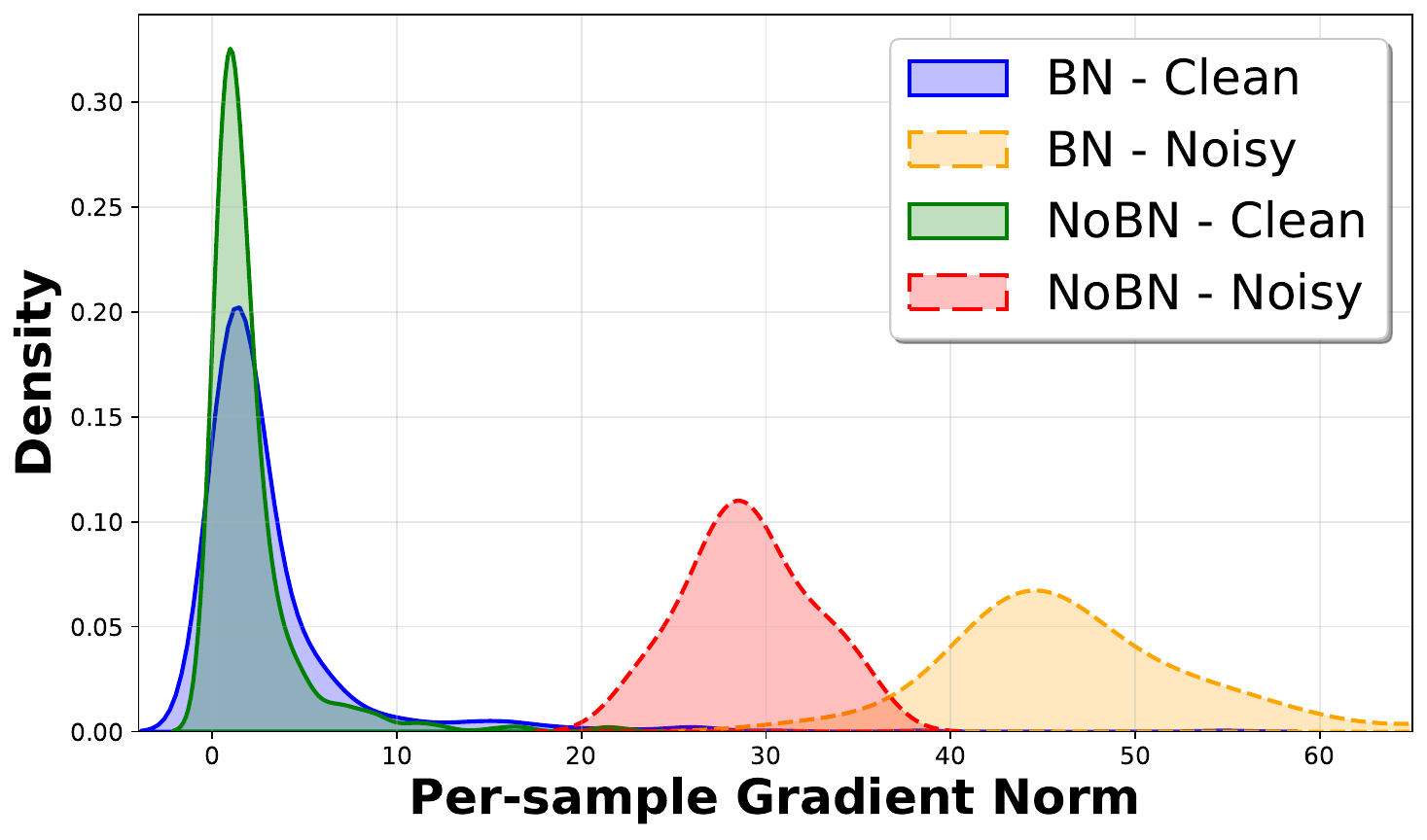}
    \caption{MNIST ($k=5\%$)}
    \label{fig:mnist_kde}
\end{subfigure}
\hfill
\begin{subfigure}{0.24\textwidth}
    \centering
    \includegraphics[width=\linewidth]{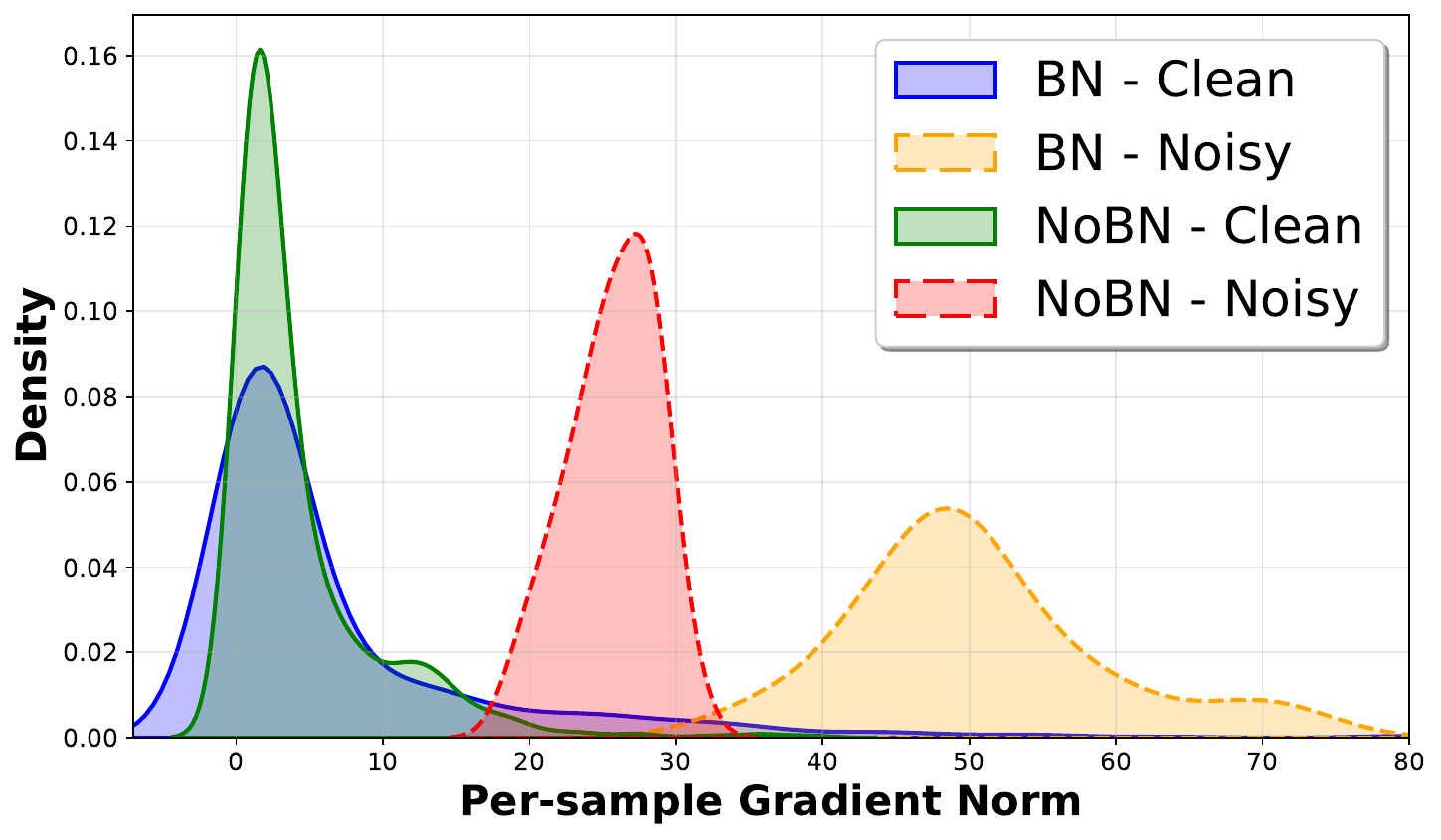}
    \caption{F-MNIST ($k=5\%$)}
    \label{fig:fashion_kde}
\end{subfigure}
\hfill
\begin{subfigure}{0.24\textwidth}
    \centering
    \includegraphics[width=\linewidth]{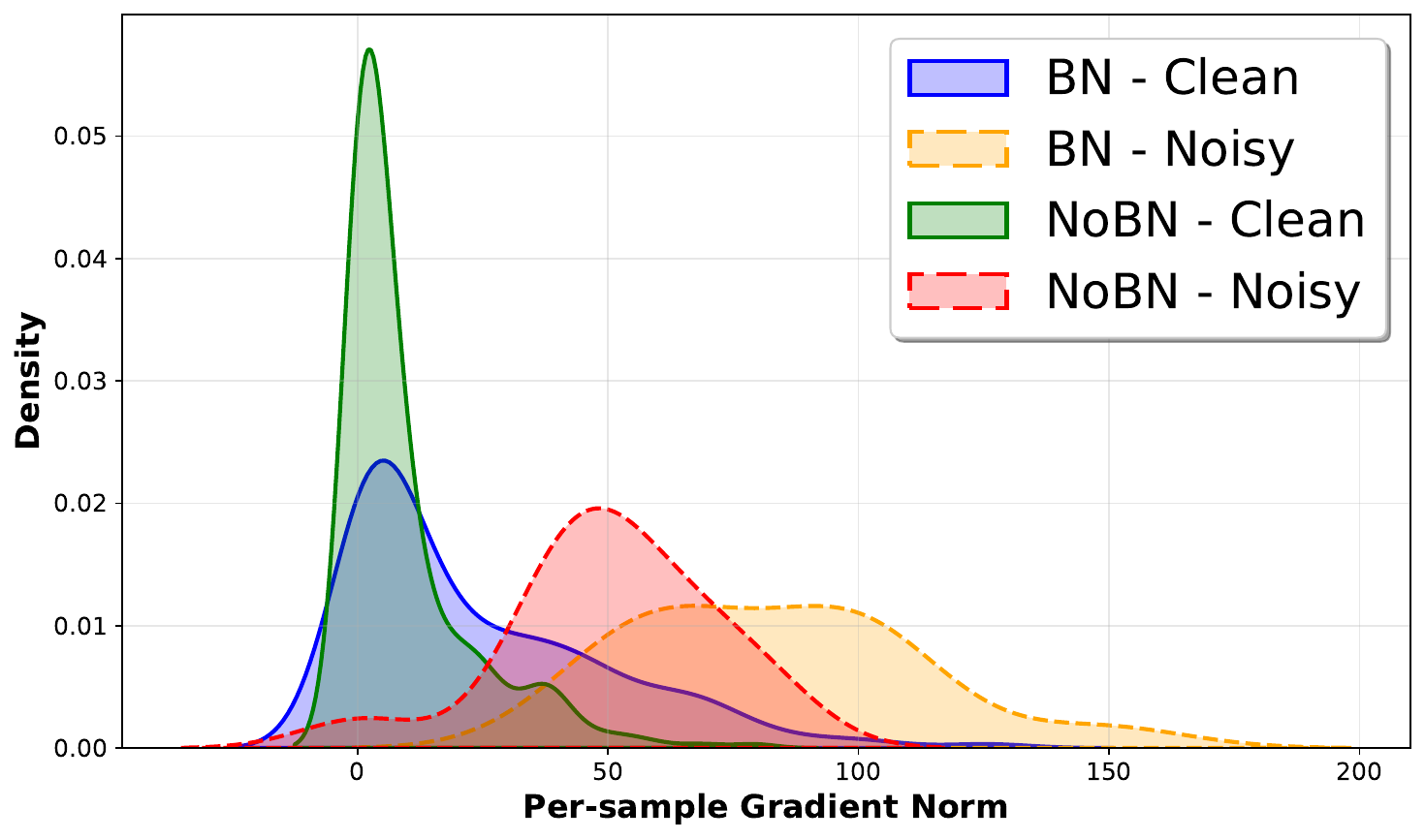}
    \caption{CIFAR10 ($k=5\%$)}
    \label{fig:cifar10_kde}
\end{subfigure}
\hfill
\begin{subfigure}{0.24\textwidth}
    \centering
    \includegraphics[width=\linewidth]{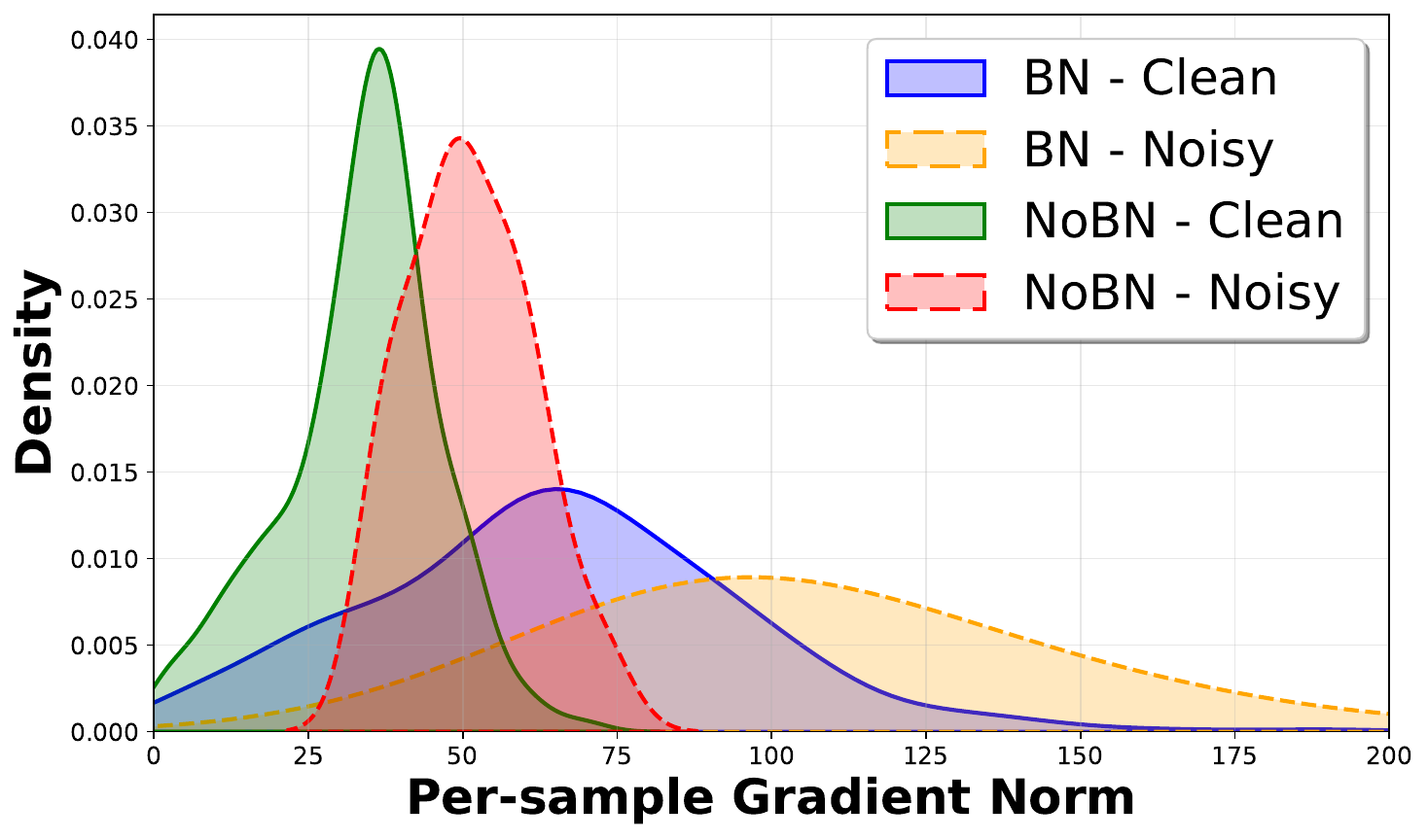}
    \caption{CIFAR100 \textcolor{red}{($k=5\%$)}}
    \label{fig:cifar100_kde}
\end{subfigure}
\hfill
\begin{subfigure}{0.24\textwidth}
    \centering
    \includegraphics[width=\linewidth]{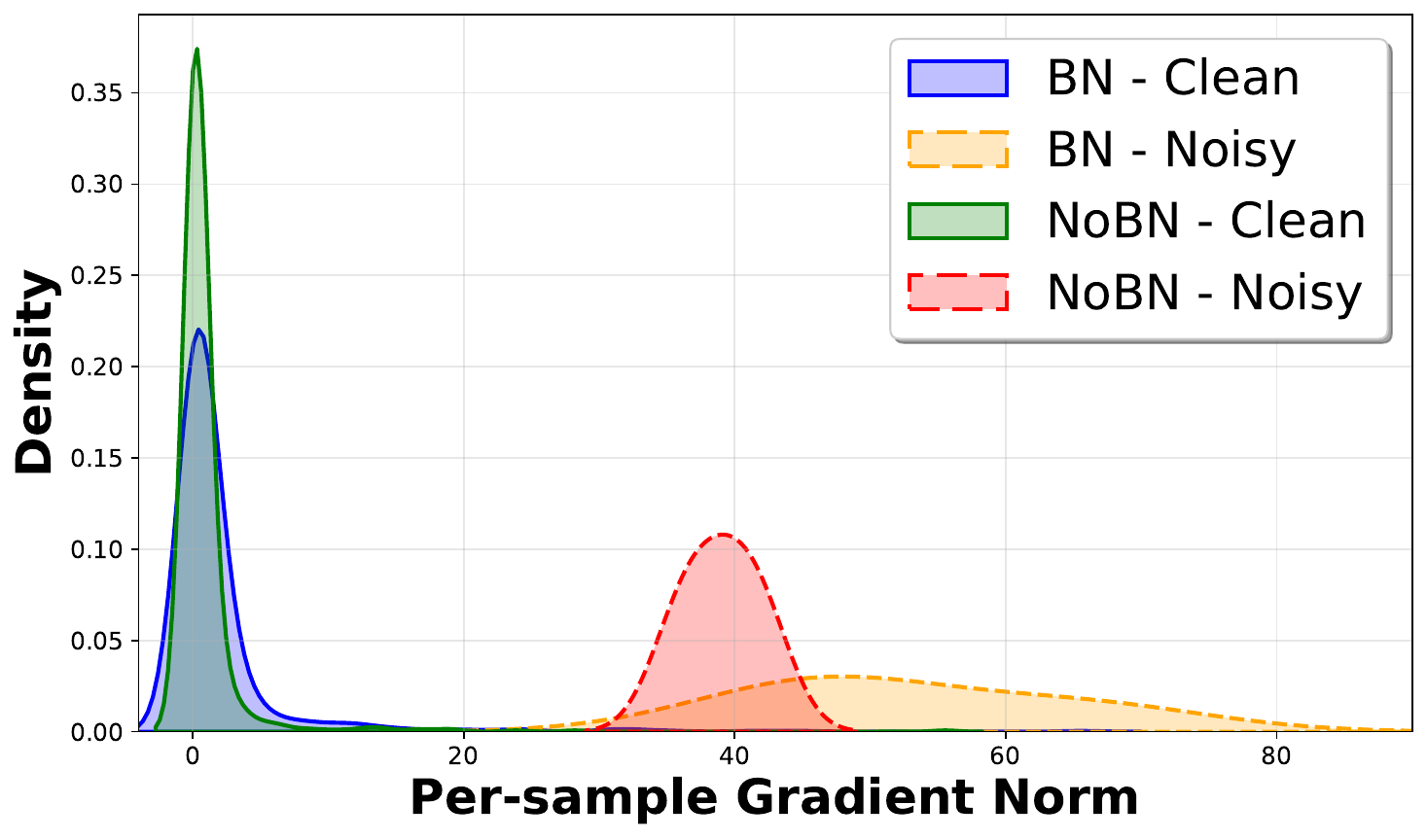}
    \caption{MNIST ($k=1\%$)}
    \label{fig:mnist_kde1}
\end{subfigure}
\hfill
\begin{subfigure}{0.24\textwidth}
    \centering
    \includegraphics[width=\linewidth]{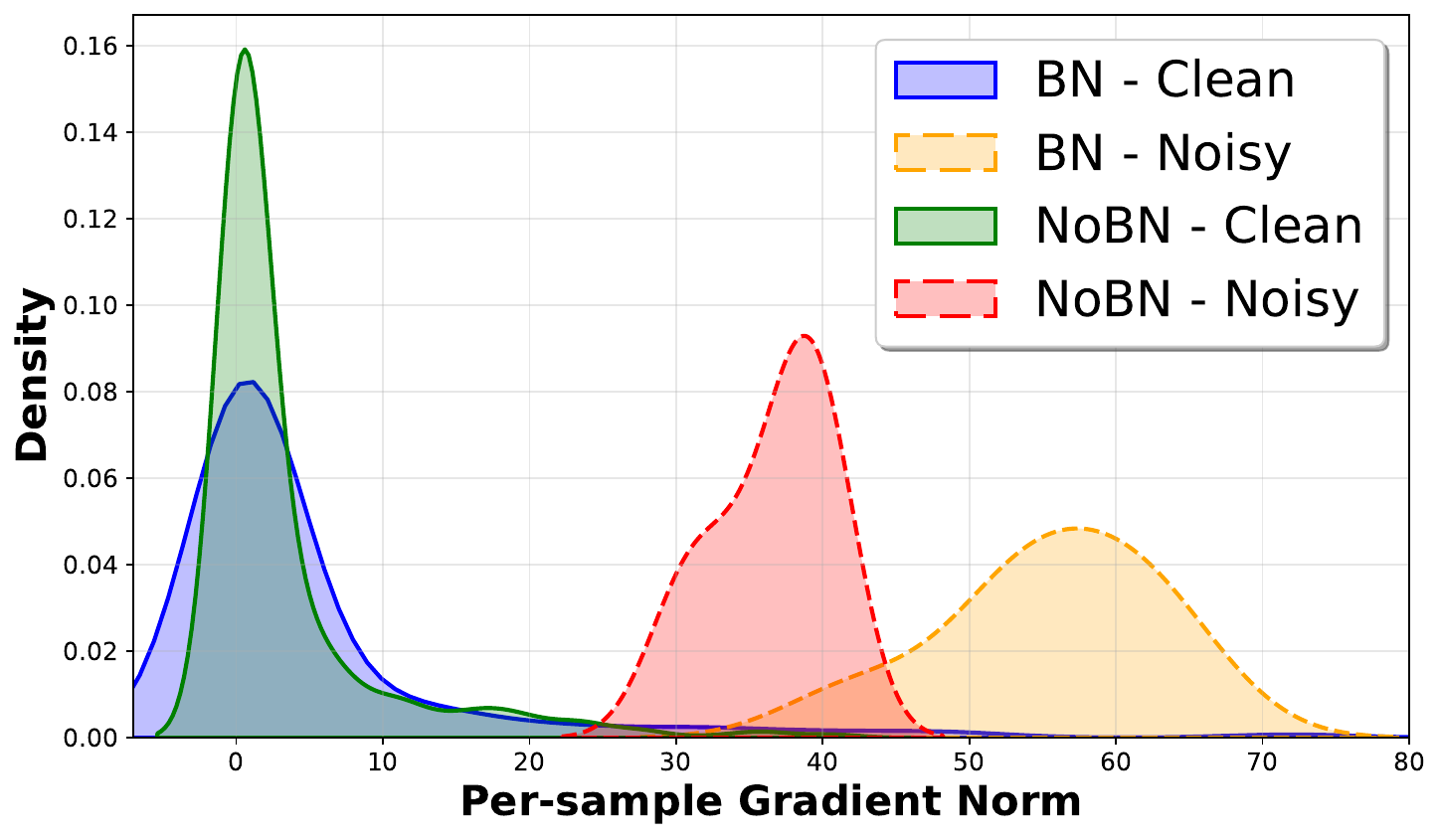}
    \caption{F-MNIST ($k=1\%$)}
    \label{fig:fashion_kde1}
\end{subfigure}
\hfill
\begin{subfigure}{0.24\textwidth}
    \centering
    \includegraphics[width=\linewidth]{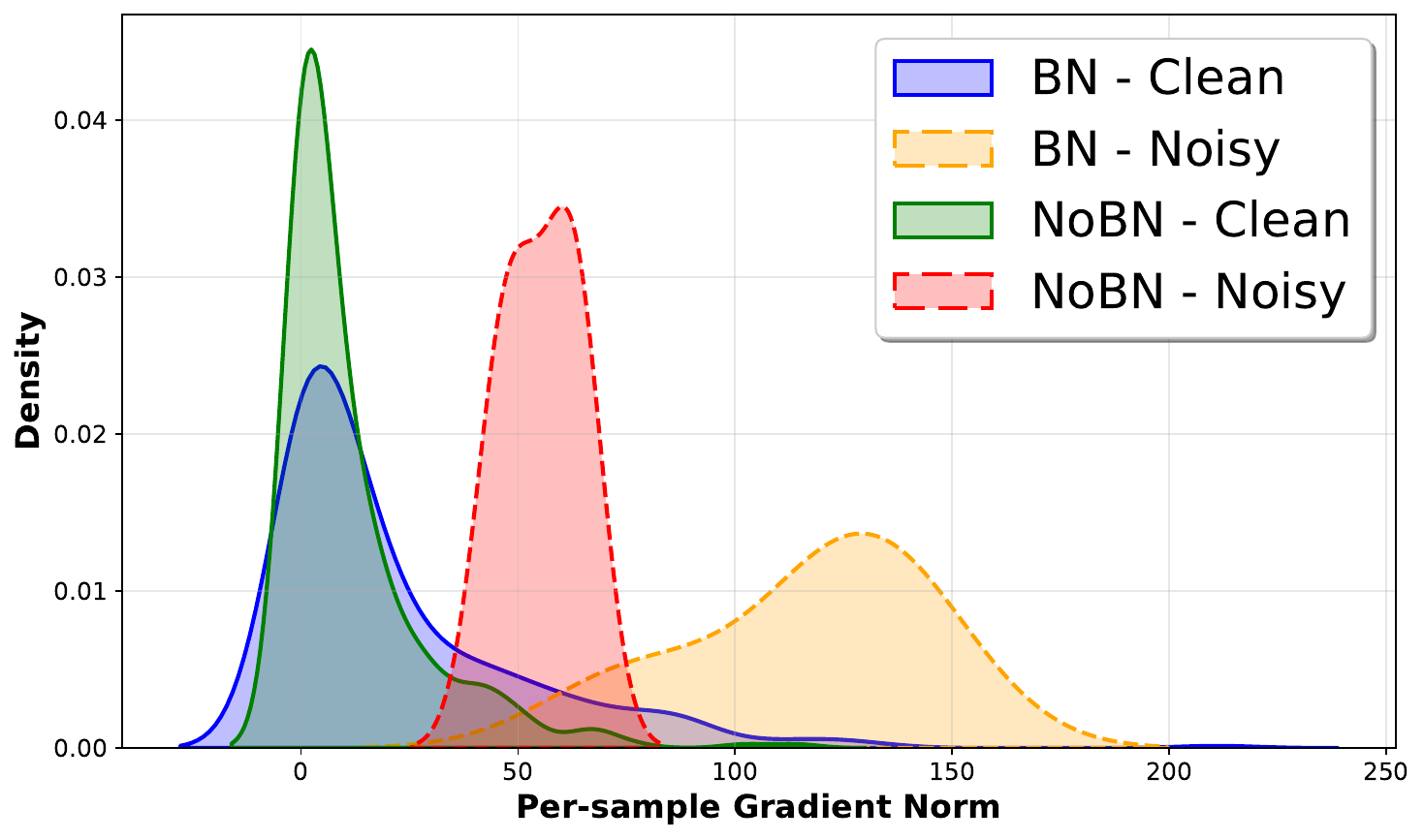}
    \caption{CIFAR10 ($k=1\%$)}
    \label{fig:cifar10_kde1}
\end{subfigure}
\hfill
\begin{subfigure}{0.24\textwidth}
    \centering
    \includegraphics[width=\linewidth]{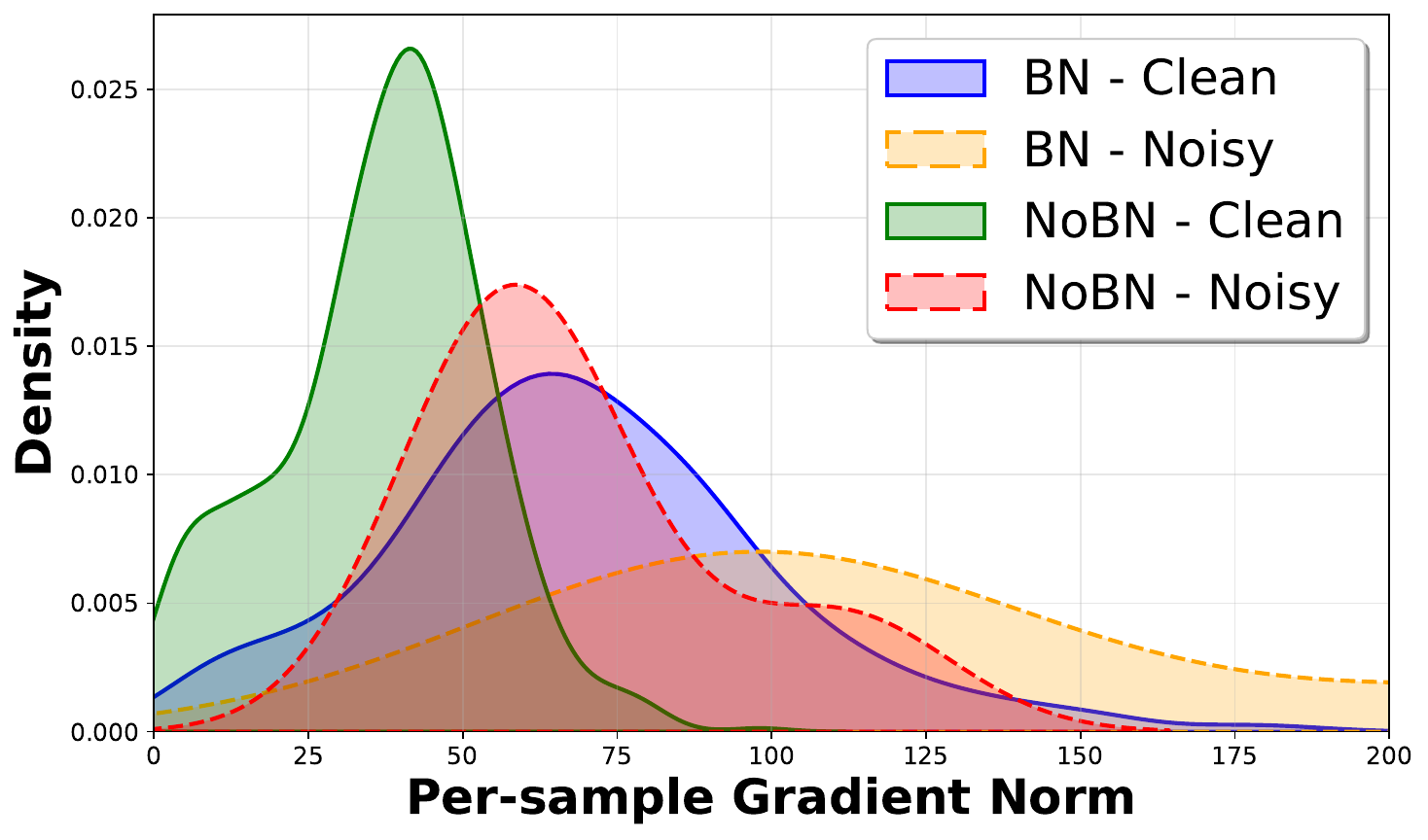}
    \caption{CIFAR100 \textcolor{red}{($k=1\%$)}}
    \label{fig:cifar100_kde1}
\end{subfigure}
\caption{Per-sample gradient norm distributions across different datasets. Each subfigure shows kernel density estimates of gradient norms for clean and noisy samples, comparing models with and without batch normalization. The x-axis represents the per-sample gradient norm, and the y-axis represents density. All datasets contain $5\%$ (top) and $1\%$ (bottom) label noise. (Best viewed in color)}
\label{fig:grad_dist}
\end{figure*}

\begin{figure*}[t]
    \centering
    \includegraphics[width=\textwidth]{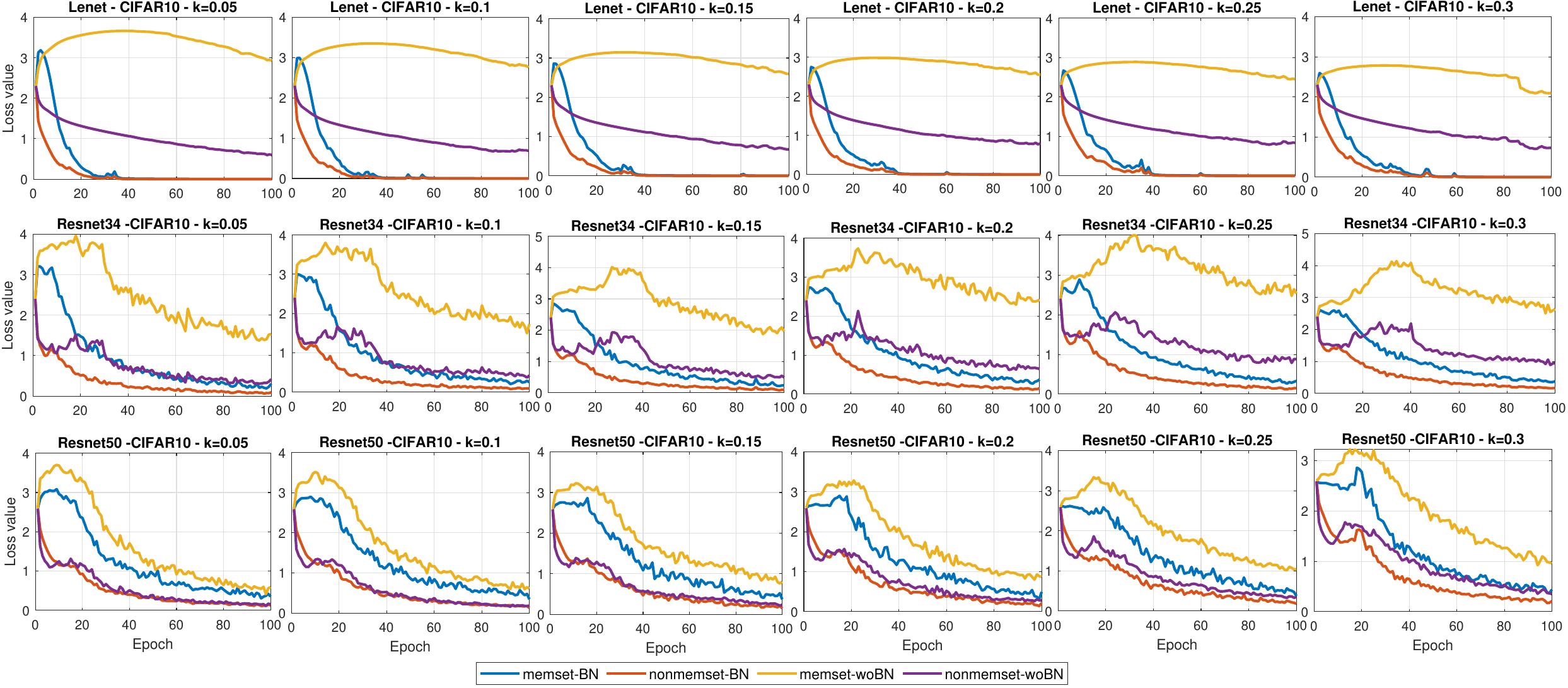}
    \caption{Loss values (a mean of five repetitions) by epoch on a memset (the original data) and a nonmemset (the flipped data) of models using BN and without using BN.}
    \label{fig:ap:loss}
\end{figure*}

\end{document}